%% file: LASSO_paper_Main.tex
\documentclass[10pt,a4paper]{article}

\usepackage{amssymb,amsmath,amsthm,amsfonts,dsfont}

\usepackage[breaklinks=true]{hyperref}
\usepackage{url}
\usepackage[dvips]{epsfig}
\usepackage[latin1]{inputenc}
\usepackage{algorithmic} 
\usepackage{algorithm} 
\usepackage[T1]{fontenc}
\usepackage{multirow}
\usepackage{fancyhdr}
\usepackage{bm}
\usepackage{subfigure}
\usepackage{authblk}

\usepackage{pgfplots}
\usepackage{etex}
\reserveinserts{18}
\usepackage{morefloats}
\usepackage[misc]{ifsym}

\usepackage{array}
\usepackage{colortbl}

\usepackage{atveryend}
\makeatletter
\let\origcitation\citation
\AtEndDocument{\def\mycites{}%
  \def\citation#1{\g@addto@macro\mycites{#1^^J}\origcitation{#1}}}
\AtVeryEndDocument{\newwrite\citeout\immediate\openout\citeout=\jobname.cit
  \immediate\write\citeout{\mycites}\immediate\closeout\citeout}
\makeatother

\usepackage{tikz}
\usepackage{tikz-3dplot}
\usetikzlibrary{fadings,calc}

\tikzfading[name=radialfade, inner color=transparent!0, outer color=transparent!100]

\DeclareMathOperator*{\argmax}{argmax}
\DeclareMathOperator*{\argmin}{argmin}

\newcommand{\cO}{\ensuremath{\mathcal{O}}}

\newcommand{\R}{\ensuremath{\mathbb{R}}}

\newcommand{\e}{\mbox{e}}
\newcommand\T{\rule{0pt}{2.3ex}}
\newcommand\B{\rule[-1.0ex]{0pt}{0pt}}
\newcommand{\quotes}[1]{``#1''}

\newcommand{\samplesize}{\kappa}
\newcommand{\activeset}{\ensuremath{\cS^{\ast}}}
\newcommand{\randomset}{\ensuremath{\cS}}
\newcommand{\activesize}{\ensuremath{s}}

\newcommand{\ball}{\ensuremath{\odot_\delta}}

\newcommand{\be}{\mathbf{e}}
\newcommand{\bu}{\mathbf{u}}

\newcommand{\cS}{\ensuremath{\mathcal{S}}}

\newtheorem{lemma}{Lemma}
\newtheorem{proposition}{Proposition}
\newtheorem{theorem}{Theorem}

\title{Fast and Scalable Lasso via Stochastic Frank-Wolfe Methods with a Convergence Guarantee}

\author[1]{Emanuele Frandi}
\author[2]{Ricardo \~Nanculef}
\author[3]{Stefano Lodi}
\author[4]{Claudio Sartori}
\author[5]{Johan A. K. Suykens}
\affil[1,5]{\small ESAT-STADIUS, KU Leuven, Belgium \texttt{\{efrandi,johan.suykens\}@esat.kuleuven.be}}
\affil[2]{\small Department of Informatics, Federico Santa Mar\'ia Technical University, Chile \texttt{jnancu@inf.utfsm.cl}}
\affil[3,4]{\small Department of Computer Science and Engineering, University of Bologna, Italy \texttt{\{stefano.lodi,claudio.sartori\}@unibo.it}}
\date{}

\begin{document}

\maketitle

\begin{abstract}
Frank-Wolfe (FW) algorithms have been often proposed over the last few years as efficient solvers for a variety of optimization problems arising in the field of Machine Learning. The ability to work with cheap projection-free iterations and the incremental nature of the method 
make FW a very effective choice for many large-scale problems where computing a sparse model is desirable. 

In this paper, we present a high-performance implementation of the FW method tailored to solve large-scale Lasso regression problems, based on a randomized iteration, and prove that the convergence guarantees of the standard FW method are preserved in the stochastic setting.
We show experimentally that our algorithm outperforms several existing state of the art methods, including the Coordinate Descent algorithm by Friedman \textit{et al.} (one of the fastest known Lasso solvers), on several benchmark datasets with a very large number of features, without sacrificing the accuracy of the model. 
Our results illustrate 
that the algorithm is able to generate the complete regularization path on problems of size up to four million variables in less than one minute.
\end{abstract}

\input{LASSO_paper_Sec1_Introduction.tex}

\input{LASSO_paper_Sec2_Lasso_Problem}

\input{LASSO_paper_Sec3_Frank_Wolfe_Optimization}

\input{LASSO_paper_Sec4a_Randomized_FW_Lasso}

\input{LASSO_paper_Sec4b_Convergence_Proof}

\input{LASSO_paper_Sec5_Experiments}

\input{LASSO_paper_Sec6_Conclusions}

\section*{Acknowledgments}
The research leading to these results has received funding from the European Research Council under the European Union's Seventh Framework Programme (FP7/2007-2013) / ERC AdG A-DATADRIVE-B (290923). This paper reflects only the authors' views and the Union is not liable for any use that may be made of the contained information. Research Council KUL: GOA/10/09 MaNet, CoE PFV/10/002 (OPTEC),
BIL12/11T; Flemish Government: FWO: projects: G.0377.12 (Structured systems), G.088114N (Tensor based data similarity); PhD/Postdoc grants; iMinds Medical Information Technologies SBO 2014; IWT:  POM II SBO 100031; Belgian Federal Science Policy Office: IUAP P7/19 (DYSCO, Dynamical systems, control and optimization, 2012-2017).
The second author received funding from CONICYT Chile through FONDECYT Project 11130122.
The first author thanks the colleagues from the Department of Computer Science and Engineering, University of Bologna, for their hospitality during the period in which this research was conceived.


\bibliographystyle{plain}      
\bibliography{LASSO_bibliography}   

\input{LASSO_paper_Appendix_new}

\end{document}

%% file: LASSO_paper_Sec1_Introduction.tex
\section{Introduction}

Many Machine Learning and Data Mining tasks can be formulated, at some stage, in the form of an optimization problem. As constantly growing amounts of high dimensional data are becoming available in the Big Data era, a fundamental thread in research is the developement of high-performance implementations of algorithms tailored to solving these problems in a very large-scale setting.
One of the most popular and powerful techniques for high-dimensional data analysis 
is the \emph{Lasso} \cite{tibshirani96}.  
%
In the last decade there has been intense interest in this method, and several papers describe generalizations and variants of the Lasso \cite{tibshirani2011}. 
In the context of supervised learning, it was recently proved that the Lasso problem can be reduced to an equivalent SVM formulation, which potentially allows one to leverage a wide range of efficient algorithms devised for the latter problem \cite{jaggi_lasso}.
For unsupervised learning, the idea of {Lasso regression} has been used in~\cite{lasso_biclustering} for bi--clustering in biological research.

From an optimization point of view, the Lasso can be formulated as an $\ell_1$-regularized least squares problem, and large-scale instances must usually be tackled by means of an efficient first-order algorithm. Several such methods have already been discussed in the literature. Variants of Nesterov's Accelerated Gradient Descent, for example, guarantee an optimal convergence rate among first-order methods \cite{nesterov}. Stochastic algorithms such as Stochastic Gradient Descent and Stochastic Mirror Descent have also been proposed for the Lasso problem \cite{langford2009sparse,shwartz2011}. More recently, Coordinate Descent (CD) algorithms \cite{friedman2007,friedman2010}, along with their stochastic variants \cite{shwartz2011,richtarik2014}, are gaining popularity due to their efficiency on structured large-scale problems. In particular, the CD implementation of Friedman \textit{et al.} mentioned above is specifically tailored for Lasso problems, and is currently recognized as one of the best solvers for this class of problems.

The contribution of the present paper in this context can be summarized as follows:
\begin{itemize}

\item We propose a high-performance stochastic implementation of the classical Frank-Wolfe (FW) algorithm to solve the Lasso problem. We show experimentally how the proposed method is able to efficiently scale up to problems with a very large number of features, improving on the performance of other state of the art methods such as the Coordinate Descent algorithm in \cite{friedman2010}. 

\item We include an analysis of the complexity of our algorithm, and prove a novel convergence result that yields an $\cO(1/k)$ convergence rate analogous (in terms of expected value) to the one holding for the standard FW method. 

\item We highlight how the properties of the FW method allow to obtain solutions that are significantly more sparse in terms of the number of features compared with those from various competing methods, while retaining the same optimization accuracy.

\end{itemize}
On a broader level, and in continuity with other works from the recent literature \cite{SWAP_paper,harchaoui14,SSCI}, the goal of this line of research is to show how FW algorithms provide a general and flexible optimization framework, encompassing several fundamental problems in Machine Learning. 

\subsubsection*{Structure of the Paper}
In Section~\ref{lasso-sec}, we provide an overview of the Lasso problem and its formulations, and review some of the related literature. Then, in Section~\ref{fw-opt}, we discuss FW optimization and specialize the algorithm for the Lasso problem. The randomized algorithm used in our implementation is discussed in Section~\ref{fw_lasso_sec}, and its convergence properties are analyzed. In Section~\ref{experiments} we show several experiments on benchmark datasets and discuss the obtained results. Finally, Section~\ref{conclusions} closes the paper with some concluding remarks.

%% file: LASSO_paper_Sec2_Lasso_Problem.tex

\section{The Lasso Problem}\label{lasso-sec}

Suppose we are given data points $(x_\ell, y_\ell)$, $\ell=1,2,\ldots,m$, where $x_\ell=(x_{\ell1},x_{\ell2},$ $\ldots,x_{\ell p})^T \in \mathbb{R}^p$ 
are some predictor variables and $y_\ell$ the respective responses. 
A common approach in statistical modeling and Data Mining is the linear regression model, which predicts $y_\ell$ as a linear combination of the input attributes: 
\[
\hat y_\ell = \sum_{i=1}^{p} \alpha_i x_{\ell i} + \alpha_0 \ .
\]
In a high-dimensional, low sample size setting ($p \gg m$), estimating the coefficient vector $\alpha=(\alpha_1,\alpha_2,\ldots,\alpha_p)^T \in \mathbb{R}^p$ using ordinary least squares leads to an ill-posed problem, i.e. the solution becomes not unique and unstable. 
In this case, a widely used approach for model estimation 
is regularization. Among regularization methods, the Lasso is of particular interest due to its ability to perform variable selection and thus obtain more interpretable models \cite{tibshirani96}. 

\subsection{Formulation}

Let $X$ be the $m \times p$ design matrix where the data is arranged row-wise, i.e. $X = [x_1,\ldots,x_m]^T$. 
Similarly, let $y = (y_1,\ldots,y_m)^T$ be the $m$-dimensional response vector. Without loss of generality, we can assume $\alpha_0 = 0$ (e.g. by centering the training data such that each attribute has zero mean) \cite{tibshirani96}. The Lasso estimates the coefficients as the solution to the following problem
\begin{equation}\label{eq:LASSO}
\underset{\alpha}{\min} \ f(\alpha) = \tfrac{1}{2} \ \| X\alpha - y \|_2^2 \ \ \mbox{s.t.} \ \ \| \alpha \|_{1} \leq \delta \ ,
\end{equation}
where the $\ell_1$-norm constraint $\| \alpha \|_{1} = \sum_i |\alpha_i| \leq \delta$ has the role of promoting sparsity in the regression coefficients. It is well-known that the constrained problem (\ref{eq:LASSO}) is equivalent to the unconstrained problem
\begin{equation}\label{eq:LASSO-unc}
\underset{\alpha}{\min} \ \tilde f(\alpha) = \tfrac{1}{2} \ \| X\alpha - y \|_2^2 + \lambda \| \alpha \|_{1} \ ,
\end{equation}
in the sense that given a solution $\alpha^*$ of (\ref{eq:LASSO-unc}) with a certain value for parameter $\bar \lambda$, it is possible to find a $\bar \delta$ such that $\alpha^*$ is also a solution of (\ref{eq:LASSO}), and vice versa. 
Specifically, if one has an exact solution $\alpha^*$ of problem (\ref{eq:LASSO-unc}), corresponding to a given $\bar \lambda$, it is easy to see that $\alpha^*$ is also a solution of (\ref{eq:LASSO}) for $\bar \delta = \|\alpha^*\|_1$. It is immediate to see that $\lambda = 0$ corresponds to the unconstrained solution $\alpha^R$, i.e. to the plain least-squares regression, which can be obtained by setting $\delta > \|\alpha^R\|_1$ in (\ref{eq:LASSO}). On the opposite, $\delta = 0$ in (\ref{eq:LASSO}) corresponds to the null solution, which is obtained for large enough values of $\lambda$ (specifically, for the case $p > m$, it can be shown that the solution of (\ref{eq:LASSO-unc}) is $\alpha^{*} = 0$ whenever $\lambda > \| X^T y \|_{\infty}$ \cite{turlach2005}).
Since the optimal tradeoff between the sparsity of the model and its predictive power is not known a-priori, practical applications of the Lasso require to find a solution and the profiles of estimated coefficients for a range of values of the regularization parameter $\delta$ (or $\lambda$ in the penalized formulation). This is known in the literature as the Lasso \textit{regularization path} \cite{friedman2010}. 

\subsection{Relevance and Applications}

The Lasso is part of a powerful family of regularized linear regression methods for high-dimensional data analysis, which also includes ridge regression (RR) \cite{hoerl1970ridge,hastie01statisticallearning}, 
the ElasticNet \cite{elasticnet}, and several recent extensions thereof \cite{zou2006adaptive,zou2009adaptive,tibshirani2005sparsity}. From a statistical point of view, they can be viewed as methods for trading off the bias and variance of the coefficient estimates in order to find a model with better predictive performance. From a Machine Learning perspective, they allow to adaptively control the capacity of the model space in order to prevent overfitting. In contrast to RR, which is obtained by substituting the $\ell_1$ norm in (\ref{eq:LASSO-unc}) by the squared $\ell_2$ norm $\sum_i |\alpha_i|^2$, it is well-known that the Lasso does not only reduce the variance of coefficient estimates but is also able to perform variable selection by shrinking many of these coefficients to zero. Elastic-net regularization trades off $\ell_1$ and $\ell_2$ norms using a ``mixed'' penalty $\Omega(\alpha) = \gamma \| \alpha \|_{1} + (1- \gamma) \| \alpha \|_{2}$ which requires tuning the additional parameter $\gamma$ \cite{elasticnet}. $\ell_p$ norms 
with $p \in [0,1)$ can enforce a more aggressive variable selection, but lead to computationally challenging non-convex optimization problems. For instance, $p=0$, which corresponds to ``direct'' variable selection, leads to an NP-hard problem \cite{weston2003}. 

Thanks to its ability to perform variable selection and model estimation simultaneously, the Lasso is used in many fields involving high-dimensional data. These scenarios have become of increasing importance in the last decades since, as technologies for collecting and processing data evolve, classification and regression problems with a large number of candidate predictors have become ubiquitous. Advances in molecular technologies, for example, enable scientists to measure the status of thousands to millions of biomolecules simultaneously \cite{gui2005penalized}. In text analysis, vector space models for representing documents easily leads to several thousands or even millions of document-term counts that correspond to potentially informative variables \cite{E2006-paper,Ifrim2008}. Similarly, in the analysis of functional magnetic resonance imaging (fMRI) data, one can easily obtain datasets with millions of \textit{voxels} representing the activity of particular portions of the brain \cite{li2009voxel}. In all these cases, the number dimensions or attributes $p$ can far exceed the number of data instances $m$. 
It is also worth mentioning that popular data analysis tools such as $\ell_1$-regularized logistic regression and penalized Cox regression models can be implemented using iterative algorithms which solve Lasso problems at each iteration \cite{friedman2010}. 

\subsection{Related Work}

Problem (\ref{eq:LASSO}) is a quadratic programming problem with a convex constraint, which in principle may be solved using standard techniques such as interior-point methods, guaranteeing convergence in few iterations. However, the computational work required \emph{per iteration} as well as the memory demanded by these approaches make them practical only for small and medium-sized problems. A faster specialized interior point method for the Lasso was proposed in \cite{kim2007interior}, which however  compares unfavorably with the baseline used in this paper \cite{friedman2010}.

One of the first efficient algorithms proposed in the literature for finding a solution of (\ref{eq:LASSO-unc}) is the Least Angle Regression (LARS) by Efron \textit{et al.} \cite{lars_paper}. As its main advantage, LARS allows to generate the entire Lasso regularization path with the same computational cost as standard least-squares via QR decomposition, i.e. $\cO(mp^2)$, assuming $m < p$ \cite{hastie01statisticallearning}. 
At each step, it identifies the variable most correlated with the residuals of the model obtained so far and includes it in the set of ``active'' variables, moving the current iterate in a direction equiangular with the rest of the active predictors. It turns out that the algorithm we propose makes the same choice but updates the solution differently using cheaper computations. 
A similar homotopy algorithm to calculate the regularization path has been proposed in \cite{turlach2005}, which differs slightly from LARS in the choice of the search direction.

More recently, it has been shown by Friedman \textit{et al.} that a careful implementation of the Coordinate Descent method (CD) provides an extremely efficient solver \cite{friedman2007}, \cite{friedman2010}, \cite{friedman2012}, which also applies to more general models such as the ElasticNet proposed by Zou and Hastie \cite{elasticnet}. In contrast to LARS, this method cyclically chooses one variable at a time and performs a simple analytical update. The full regularization path is built by defining a sensible range of values for the regularization parameter and taking the solution for a given value as warm-start for the next. This algorithm has been implemented into the Glmnet package
and can be considered the current standard for solving this class of problems. 
Recent works have also advocated the use of Stochastic Coordinate Descent (SCD) \cite{shwartz2011}, where the order of variable updates is chosen randomly instead of cyclically. This strategy can prevent the adverse effects caused by possibly unfavorable orderings of the coordinates, and allows to prove stronger theoretical guarantees compared to the plain CD \cite{richtarik2014}.

Other methods for $\ell_1$-regularized regression may be considered. For instance, Zhou \textit{et al.} recently proposed a geometrical approach where the Lasso is reformulated as a nearest point problem and solved using an algorithm inspired by the classical Wolfe method \cite{geolasso}. However, the popularity and proved efficiency of Glmnet on high-dimensional problems make it the chosen baseline in this work. 


%% file: LASSO_paper_Sec3_Frank_Wolfe_Optimization.tex

\section{Frank-Wolfe Optimization}\label{fw-opt}
One of the earliest constrained optimization approaches  \cite{wolfe1954,wolfe1970}, the Frank-Wolfe algorithm has recently seen a sudden resurgence in interest from the Optimization and Machine Learning communities, due to its powerful theoretical properties and proved efficiency in the context of large-scale problems \cite{clarkson_coresets,Jaggi2013ICMLa,harchaoui14}. 
On the theoretical side, FW methods come with iteration complexity bounds that are independent of the number of variables in the problem, and sparsity guarantees that hold during the whole execution of the algorithm \cite{clarkson_coresets,Jaggi2013ICMLa}. In addition, several variants of the basic procedure have been analyzed, which can improve the convergence rate and practical performance of the basic FW iteration \cite{GuelatMarcotte,SWAP_paper,jaggi_linconv,PARTAN_paper}. 
From a practical point of view, they have emerged as efficient alternatives to traditional methods in several contexts, such as large-scale SVM classification \cite{CIARP,IJPRAI11,SWAP_paper,PARTAN_paper} and nuclear norm-regularized matrix recovery \cite{Jaggi2013ICMLa,SSCI}. In view of these developements, FW algorithms have come to be regarded as a suitable approach to large-scale optimization 
in various areas of Machine Learning, statistics, bioinformatics and related fields \cite{signoretto14book,Jaggi2013ICMLb}.

Overall, though, the number of works showing experimental results for FW on practical applications is limited compared to that of the theoretical studies which have appeared in the literature. In the context of problems with $\ell_1$-regularization or sparsity constraints, the use of FW has been discussed in \cite{shwartz2010}, but no experiments are provided. A closely related algorithm has been proposed in \cite{geolasso}, however its implementation has a high computational cost in terms of time and memory requirements, and is not suitable for solving large-scale problems on a standard desktop or laptop machine. As such, the current literature does not provide many examples of efficient FW-based software for large-scale Lasso or $l_1$-regularized optimization. 
In this work, we aim to fill this gap by showing how a properly implemented stochastic FW method can best the current state of the art solvers on Lasso problems with a very large number of features.


\subsection{The Standard Frank-Wolfe Algorithm}

The FW algorithm is a general method to solve problems of the form
\begin{equation} \label{eq:GENERICproblem}
\min_{\alpha \in \Sigma} \; f(\alpha),
\end{equation}
where $f: \R^p \rightarrow \R$ is a convex differentiable function, and $\Sigma \subset \R^p$ is a compact convex set. Given an initial guess ${\alpha}^{(0)} \in \Sigma$, the standard {FW} method consists of the steps outlined in Algorithm \ref{alg:FW}. 
\begin{algorithm}
\begin{algorithmic}[1]
\STATE \textbf{Input:} an initial guess $\alpha^{0}$. %
\FOR{$k = 0,1,\ldots$}
	\STATE Define a search direction $d^{(k)}$ by optimizing a linear model:
        \begin{equation}\label{linear_subpb}
        u^{(k)} 
         \in \argmin_{u \,\in\, \Sigma} \, (u-\alpha^{(k)})^T \nabla f({\alpha}^{(k)}), \qquad 
        {d}^{(k)} = {u}^{(k)} - {\alpha}^{(k)}. 
        \end{equation}
	\STATE Choose a stepsize $\lambda^{(k)}$, e.g. via line-search: 
	    \begin{equation}\label{FW_line_search}
	   \lambda^{(k)} \in  \argmin_{\lambda \,\in\, [0,1]}  f({\alpha}^{(k)} + \lambda {d}^{(k)}).
	   \end{equation}
        \STATE Update: 
        ${\alpha}^{(k+1)}=  {\alpha}^{(k)} + \lambda^{(k)} {d}^{(k)}\,.$ 
\ENDFOR
\end{algorithmic}
\caption{\label{alg:FW} {The standard Frank-Wolfe algorithm}}
\end{algorithm}
From an implementation point of view, a fundamental advantage of FW is that the most demanding part of the iteration, i.e. the solution of the linear subproblem (\ref{linear_subpb}), has a computationally convenient solution for several problems of practical interest, mainly due to the particular form of the feasible set. The key observation is that, when $\Sigma$ is a polytope 
(e.g. the unit simplex for $L_2$-SVMs \cite{SWAP_paper}, the $\ell_1$-ball of radius $\delta$ for the Lasso problem (\ref{eq:LASSO}), a spectrahedron in nuclear norm for matrix recovery \cite{Jaggi_spectrahedron}), 
the search in step $3$ can be reduced to a search among the vertices of $\Sigma$. This allows to devise cheap analytical formulas to find $u^{(k)}$, ensuring that each iteration has an overall cost of $\cO(p)$. 
The fact that FW methods work with \textit{projection-free} iterations is also a huge advantage on many practical problems, since a projection step to maintain the feasibility of the iterates (as needed by classical approaches such as proximal methods for Matrix Recovery problems) generally has a super-linear complexity, 
making the solution of large-scale problems difficult in practice \cite{signoretto14book}.

Another distinctive feature of the algorithm is the fact that the solution at a given iteration $K$ can be expressed as a convex combination of the vertices $u^{(k)}$, $k = 0,\ldots,K-1$. Due to the incremental nature of the FW iteration, at most one new extreme point of $\Sigma$ is discovered at each iteration, implying that at most $k$ of such points are active at iteration $k$. Furthermore, this sparsity bound holds for the entire run of the algorithm, effectively allowing to control the sparsity of the solution as it is being computed. 
This fact carries a particular relevance in the context of sparse approximation, and generally in all applications where it is crucial to find models with a small number of features. It also represents, as we will show in our experiments in Section \ref{experiments}, one of the major differences between incremental, forward approximation schemes and more general solvers for $\ell_1$-regularized optimization, which in general cannot guarantee to find sparse solutions along the regularization path.


\subsection{Theoretical Properties}\label{theoretical_sec}

We summarize here some well-known theoretical results for the FW algorithm which are instrumental in understanding the behaviour of the method. We refer the reader to \cite{Jaggi2013ICMLa} for the proof of the following proposition. To prove the result, it is sufficient to assume that $f$ has bounded curvature, which, as explained in \cite{Jaggi2013ICMLa}, is roughly equivalent to the Lipschitz continuity of $\nabla f$.

\begin{proposition}[Sublinear convergence, \cite{Jaggi2013ICMLa}]\label{subconv}
Let $\alpha^{*}$ be an optimal solution of problem (\ref{eq:GENERICproblem}). Then, for any $k \geq 1$, the iterates of Algorithm \ref{alg:FW} satisfy 
\begin{equation*}
f(\alpha^{(k)}) - f(\alpha^{*}) \leq \frac{4C_f}{k+2}\,,
\end{equation*}
where $C_f$ is the \textit{curvature constant} of the objective function. 
\end{proposition}

An immediate consequence of Proposition \ref{subconv} is an upper bound on the iteration complexity: given a tolerance $\varepsilon > 0$, the FW algorithms finds an \textit{$\varepsilon$-approximate solution}, i.e. an iterate $\alpha^{(k)}$ such that $f(\alpha^{(k)}) - f(\alpha^{*}) \leq \varepsilon$, after $\cO(1/\varepsilon)$ iterations.
Besides giving an estimate on the total number of iterations which has been shown experimentally to be quite tight in practice \cite{NIPS14,PARTAN_paper}, this fact tells us that the tradeoff between sparsity and accuracy can be partly controlled by appropriately setting the tolerance parameter.
Recently, Garber and Hazan showed that under certain conditions the FW algorithm can obtain a convergence rate of $\cO(1/k^2)$, comparable to that of first-order algorithms such as Nesterov's method \cite{garber2015}. However, their results require strong convexity of the objective function and of the feasible set, a set of hypotheses which is not satisfied for several important ML problems such as the Lasso or the Matrix Recovery problem with trace norm regularization. 

Another possibility is to employ a \textit{Fully-Corrective} variant of the algorithm, where at each step the solution is updated by solving the problem restricted to the currently active vertices. The algorithm described in \cite{geolasso}, where the authors solve the Lasso problem via a nearest point solver based on Wolfe's method, operates with a very similar philosophy. A similar case can be made for the LARS algorithm of \cite{lars_paper}, which however updates the solution in a different way. 
The Fully-Corrective FW also bears a resemblance to the Orthogonal Matching Pursuit algorithms used in the Signal Processing literature \cite{tropp2004}, a similarity which has already been discussed in \cite{clarkson_coresets} and \cite{Jaggi2013ICMLa}. 
However, as mentioned in \cite{clarkson_coresets}, the increase in computational cost is not paid off by a corresponding improvement in terms of complexity bounds. In fact, the work in \cite{lan14} shows that the result in Proposition \ref{subconv} cannot be improved for any first-order method based on solving linear subproblems 
without strengthening the assumptions. Greedy approximation techniques based on both the vanilla and the Fully-Corrective FW have also been proposed in the context of approximate risk minimization with an $\ell_0$ constraint by Shalev-Shwartz \textit{et al.}, who proved several strong runtime bounds for the sparse approximations of arbitrary target solutions \cite{shwartz2010}. 

Finally, it is worth mentioning that the result of Proposition \ref{subconv} can indeed be improved by using variants of FW that employ additional search directions, and allow under suitable hypotheses to obtain a linear convergence rate \cite{SWAP_paper,jaggi_linconv}. 
It should be mentioned, however, that such rates only hold in the vicinity of the solution and that, as shown in \cite{PARTAN_paper}, a large number of iterations might be required to gain substantial advantages. For this reason, we choose not to pursue this strategy in the present paper.


%% file: LASSO_paper_Sec4a_Randomized_FW_Lasso.tex
\section{Randomized Frank-Wolfe for Lasso Problems}\label{fw_lasso_sec}

A specialized FW algorithm for problem (\ref{eq:LASSO}) can be obtained straightforwardly by setting $\Sigma$ equal to the $\ell_1$-ball of radius $\delta$, hereafter denoted as $\ball$. In this case, the vertices of the feasible set (i.e., the candidate points among which $u^{(k)}$ is selected in the FW iterations) are $\mathcal{V}(\ball)=\{Ê\pm \delta \be_i \ : \ i= 1,2,\ldots,p \}$, where $\be_i$ is the $i$-th element of the canonical basis. It is easy to see that the linear subproblem (\ref{linear_subpb}) in Algorithm \ref{alg:FW} has a closed-form solution, given by:
\begin{equation}\label{linear_subpb_l1ball}
\begin{aligned}
 u^{(k)} &= - \delta \, \mbox{sign}\left(\nabla f({\alpha}^{(k)})_{i_{\ast}^{(k)}} \right) \be_{i_{\ast}^{(k)}} \equiv \tilde{\delta}^{(k)} \be_{i_{\ast}^{(k)}}, \\
   i_{\ast}^{(k)} &= \argmax_{i=1,\ldots,p} \left| \nabla f({\alpha}^{(k)})_i \right| . 
\end{aligned}
\end{equation}

In order to efficiently execute the iteration, we can exploit the form of the objective function to obtain convenient expressions to update the function value and the gradient after each FW iteration. The gradient of $f(\cdot)$ in (\ref{eq:LASSO}) is 
\begin{equation} \label{eq:gradientLASSO}
\nabla f({\alpha}) = -X^{T}\left(y-X\alpha\right) =  -X^{T}y + X^{T}X\alpha \notag \,.
\end{equation}
There are two possible ways to to compute $\nabla f({\alpha}^{(k)})_i$ efficiently. One is to keep track of the vector of residuals $R^{(k)}=\left(y-X{\alpha}^{(k)}\right) \in \R^m$ and compute $\nabla f({\alpha}^{(k)})_i$ as 
\begin{equation} \label{eq:computation_coordinate1}
\nabla f({\alpha}^{(k)})_i = -z_i^TR^{(k)} = -z_i^Ty + z_i^TX{\alpha}^{(k)} \, ,
\end{equation}
where $z_i \in \R^m$ is the $i$-th column of the design matrix $X$, i.e., the vector formed by collecting the $i$-th attribute of all the training points. We refer to this approach as the ``method of residuals''. The other way is to expand the second line in (\ref{eq:gradientLASSO}) 
\[ 
\nabla f({\alpha}^{(k)})_i = -z_i^Ty + \sum_{j \neq 0} \alpha^{(k)}_j z_i^{T} z_{j} \, ,
\]
and keep track of the inner products $z_i^{T} z_{j}$ between $z_i$ and the predictors $z_j$ corresponding to non-zero coefficients of the current iterate. We call this the ``method of active covariates''. The discussion in the next subsections reveals that the first approach is more convenient if, at each iteration, we only need to access a small subset of the coordinates of $\nabla f({\alpha}^{(k)})$. 
It is clear from (\ref{linear_subpb_l1ball}) that after computing $\nabla f({\alpha}^{(k)})_i$ for $i=1,\ldots,p$ the solution to the linear subproblem in Algorithm \ref{alg:FW} follows easily. The other quantities required by the algorithm are the objective value (in order to monitor convergence) and 
the line search stepsize in (\ref{FW_line_search}), which can be obtained as
\begin{equation} \label{eq:obj_and_line_search}
\begin{aligned}
f(\alpha^{(k)}) &= \frac{1}{2} y^Ty + \frac{1}{2} S^{(k)} - F^{(k)}, \\
 \lambda^{(k)} &= \lambda_\ast := \frac{S^{(k)} - \tilde{\delta} \nabla f({\alpha}^{(k)})_{i\ast} - F^{(k)}}{S^{(k)} - 2 \tilde{\delta} G_{i\ast} +  \tilde{\delta}^2 z_{i\ast}^{T}z_{i\ast}} \, , 
\end{aligned}
\end{equation}
where $i\ast = i_{\ast}^{(k)}$, 
$ G_{i\ast}=\nabla f(\alpha^{(k)})_{i\ast} +  z_{i\ast}^Ty$, and the terms $S^{(k)}, F^{(k)}$ can be updated recursively as 
\begin{equation*}
\begin{aligned}
S^{(k+1)} &= (1-\lambda_\ast)^2 S^{(k)} + 2 \tilde{\delta} \lambda_\ast (1-\lambda_\ast) G_{i\ast}+ \tilde{\delta}^2\lambda_\ast^2 z_{i\ast}^{T}z_{i\ast} \notag\\
F^{(k+1)} &= (1-\lambda_\ast) F^{(k)} + \tilde{\delta}\lambda_\ast z_{i\ast}^{T}y \ , 
\end{aligned}
\end{equation*}
with starting values $S^{(0)}=0$ and $F^{(0)} = 0$. If we store the products $z_{i}^{T}y$ before the execution of the algorithm, the only non-trivial computation required here is $\nabla f({\alpha}^{(k)})_{i\ast}$ which was already computed to solve the subproblem in (\ref{linear_subpb_l1ball}).

\subsection{Randomized Frank-Wolfe Iterations}

Although the FW method is generally very efficient for structured problems with a sparse solution, it also has a number of practical drawbacks. For example, it is well known that the total number of iterations required by a FW algorithm can be large, thus making the optimization prohibitive on very large problems. Even when (\ref{linear_subpb}) has an analytical solution due to the problem structure, the resulting complexity depends on the problem size 
\cite{SWAP_paper}, and can thus be impractical in cases where handling large-scale datasets is required. 
For example, in the specialization of the algorithm to problem (\ref{eq:LASSO}), the main bottleneck is the computation of the FW vertex $i_{\ast}^{(k)}$ in (\ref{linear_subpb_l1ball}) which corresponds to examining all the $p$ candidate predictors and choosing the one most correlated with the current residuals (assuming the design matrix has been standardized s.t. the predictors have unit norm). Coincidentally, this is the same strategy underlying well-known methods for variable selection such as LARS and Forward Stepwise Regression (see Section $1$). 

A simple and effective way to avoid this dependence on $p$ is to compute the FW vertex approximately, by limiting the search to a fixed number of extreme points on the boundary of the feasible set $\Sigma$ \cite{Smola01Learning,NIPS14}. Specialized to the Lasso problem (\ref{eq:LASSO}), this technique can be formalized as extracting a random sample $\mathcal{S} \subseteq \{1,\ldots,p\}$ and solving
\begin{equation}\label{randomized_it_lasso}
u^{(k)} = - \delta \, \mbox{sign}\left(\nabla f({\alpha}^{(k)})_{i^{(k)}_{\cS}} \right) \be_{i^{(k)}_{\cS}}, \qquad \mbox{where }\; i^{(k)}_{\cS} = \argmax_{i\in \cS} \left| \nabla f({\alpha}^{(k)})_i \right| . 
\end{equation}
Formally, one can think of a randomized FW iteration as the optimization of an approximate linear model, built by considering the partial gradient $\nabla f(\alpha^{(k)})_{|\cS^{(k)}}$, i.e. the projection of the gradient onto the subspace identified by the sampled coordinates \cite{stochastic_FW_2014}.
The number of coordinates of the gradient that need to be estimated with this scheme is $|\mathcal{S}|$ instead of $p$. If $|\mathcal{S}| \ll p$, this leads to a significant reduction in terms of computational overhead. Our stochastic specialization of the FW iteration for the Lasso problem takes thus the form of Algorithm \ref{alg:LASSO-FW}. After selecting the variable $z_{i\ast} \in \cS$ best correlated with the current vector of residuals, 
the algorithm updates the current solution along the segment connecting $z_{i\ast} \in \cS$ with $\alpha^{(k)}$. Note how this approach differs from a method like LARS, where the direction to move the last iterate is equiangular to all the active variables. 
\footnote{As shown in \cite{hastie01statisticallearning}, the LARS direction is $d_k = (X_{A_k}^{T}X_{A_k})^{-1}X_{A_k}^T R^{(k)}$ where $X_{A_k}$ is the restriction of the design matrix to the active variables. The FW direction is just $d_k = e_{i\ast} - \alpha^{(k)}$.} 
It also differs from CD, which can make active more than one variable at each ``epoch'' or cycle through the predictors. The algorithm computes the stepsize by looking explicitly to the value of the objective, which can be computed analytically without increasing the cost of the iteration. Finally, the method updates the vector of residuals and proceeds to the next iteration. 

\begin{algorithm}
\begin{algorithmic}[1]
\STATE Choose the sampling set $\cS$ (see Section \ref{sampling_sec}).
\STATE Search for the predictor best correlated with the vector of residuals $R^{(k)} = \left(y-X{\alpha}^{(k)}\right)$:
\begin{equation*}
i_{\ast}^{(k)} =  \argmax_{i\in \cS} \left| \nabla f({\alpha}^{(k)})_i \right| \equiv  \left|z_i^TR^{(k)}\right| \,.
\end{equation*}
\STATE Set $\tilde{\delta}^{(k)}  = - \delta \, \mbox{sign}\left(\nabla f({\alpha}^{(k)})_{i_{\ast}} \right) $.
\STATE Compute the step-size $\lambda^{(k)}$ using (\ref{eq:obj_and_line_search}).
\STATE Update the vector of coefficients as
\[
{\alpha}^{(k+1)}= (1-\lambda^{(k)}){\alpha}^{(k)} + \tilde{\delta} \lambda^{(k)} \be_{i_{\ast}^{(k)}} \, .
\]
\STATE Update the vector of residuals $R^{(k)}$
\begin{equation}\label{eq:updating-residuals}
R^{(k+1)}= (1-\lambda^{(k)})R^{(k)} + \lambda^{(k)}\left(y - \tilde{\delta}z_{i_{\ast}^{(k)}}\right) \, .
\end{equation}
\end{algorithmic}
\caption{\label{alg:LASSO-FW} {Randomized Frank-Wolfe step for the Lasso problem}}
\end{algorithm}

Note that, although in this work we only discuss the basic Lasso problem, extending the proposed implementation to the more general ElasticNet model of \cite{elasticnet} 
is straightforward. The derivation of the necessary analytical formulae is analogous to the one shown above. Furthermore, an extension of the algorithm to solve $\ell_1$-regularized logistic regression problems, another relevant tool in high-dimensional data analysis, can be easily obtained following the guidelines in \cite{friedman2010}. 

\subsection{Complexity and Implementation Details}

In Algorithm \ref{alg:LASSO-FW}, we compute the coordinates of the gradient using the method of residuals given by equation (\ref{eq:computation_coordinate1}). Due to the randomization, this method becomes very advantageous with respect to the use of the alternative method based on the active covariates, even for very large $p$. Indeed, if we denote by $s$ the cost of performing a dot product between a predictor $z_i$ and another vector in $\mathbb{R}^m$, the overall cost of picking out the FW vertex in step $1$ of our algorithm is $\cO(s|\mathcal{S}|)$. Using the method of the active covariates would instead give an overall cost of $\cO(s|\mathcal{S} |\|{\alpha}^{(k)}\|_0)$, which is always worse. Note however that this method may be better than the method of the residuals in a deterministic implementation by using caching tricks as proposed in \cite{friedman2007}, \cite{friedman2010}. For instance, caching the dot products between all the predictors and the active ones and keeping updated all the coordinates of the gradient would cost $\cO(p)$ except when new active variables appear in the solution, in which case the cost becomes $\cO(ps)$. However, this would allow to find the FW vertex in $\cO(p)$ operations. In this scenario, the fixed $\cO(sp)$ cost of the method of residuals may be worse if the Lasso solution is very sparse. It is worth noting that the dot product cost $s$ is proportional to the number of nonzero components in the predictors, which in typical high-dimensional problems is significantly lower than $m$.

In the current implementation, the term $\sigma_i :=z_i^Ty$ will be pre-computed for any $i=1,2,\ldots,p$ before starting the iterations of the algorithm. This allows to write (\ref{eq:computation_coordinate1}) as $-z_i^TR^{(k)} = -\sigma_i + z_i^TX{\alpha}^{(k)}$. Equation (\ref{eq:updating-residuals}) for updating residuals can therefore be replaced by an equation to update $p^{(k)}=X{\alpha}^{(k)}$, eliminating the dependency on $m$. 



\subsection{Relation to SVM Algorithms and Sparsity Certificates}
The previous implementation suggests that the FW algorithm will be particularly suited to the case $p \gg m$ where a regression problem has a very large number of  features but not so many training points. 
It is interesting to compare this scenario to the situation in SVM problems: in the SVM case, the FW vertices correspond to training points, and the standard FW algorithm is able to quickly discover the relevant \quotes{atoms} (the Support Vectors), but has no particular advantage when handling lots of features. In contrast, in Lasso applications, where we are using the $z_i$'s as training points, the situation is somewhat inverted: the algorithm should discover the critical features in at most $\cO(1/\epsilon)$ iterations and guarantee that at most $\cO(1/\epsilon)$ attributes will be used to perform predictions. This is, indeed, the scenario in which Lasso is used for several applications of practical interest, as problems with a very large number of features arise frequently in specific domains like bio-informatics, web and text analysis and sensor networks. 

In the context of SVMs, the randomized FW algorithm has been already discussed in \cite{NIPS14}. However, the results in the mentioned paper were experimental in nature, and did not contain a proof of convergence, which is instead provided in this work. Note that, although we have presented the randomized search for the specific case of problem (\ref{eq:LASSO}), the technique applies more generally to the case where $\Sigma$ is a polytope (or has a separable structure with every block being a polytope, as in \cite{Jaggi2013ICMLb}). We do not feel this hypothesis to be restrictive, as basically every practical application of the FW algorithm proposed in the literature falls indeed into this setting.




%% file: LASSO_paper_Sec4b_Convergence_Proof.tex
\subsection{Convergence Analysis}

We show that the stochastic FW converges (in the sense of expected value) with the same rate as the standard algorithm. First, we need the following technical result.
\vspace{0.4cm}
\input{lemma_1_proof}

Note that selecting a random subset $\mathcal{S}$ of size $\samplesize$ to solve (\ref{randomized_it_lasso}) is equivalent to (i) building a random matrix $A_\mathcal{S}$ as in Lemma \ref{tech_lemma}, (ii) computing the restricted gradient $\tilde{\nabla}_{\cS} f =  \frac{p}{\samplesize} A_\mathcal{S} \nabla f({\alpha}^{(k)})$ and then (iii) solving the linear sub-problem (\ref{linear_subpb_l1ball}) substituting $\nabla f({\alpha}^{(k)})$ by $\tilde{\nabla}_{\cS} f$. In other words, the proposed randomization can be viewed as approximating $\nabla f({\alpha}^{(k)})$ by $\tilde{\nabla}_{\cS} f$. Lemma \ref{tech_lemma} implies that $\mathbb{E}[\tilde{\nabla}_{\cS} f] =  \nabla f({\alpha}^{(k)})$, which is the key to prove our main result. 

\begin{proposition}\label{prop2}
Let $\alpha^{*}$ be an optimal solution of problem (\ref{eq:GENERICproblem}). Then, for any $k \geq 1$, the iterates of Algorithm \ref{alg:FW} with the randomized search rule satisfy 
\begin{equation*}
\mathbb{E}_{\cS^{(k)}}[f(\alpha^{(k)})] - f(\alpha^{*}) \leq \frac{4\tilde C_f}{k+2}\,,
\end{equation*}
where $\mathbb{E}_{\cS^{(k)}}$ denotes the expectation with respect to the $k$-th random sampling.
\end{proposition}

This result has a similar flavor to that in \cite{Jaggi2013ICMLb}, and the analysis is similar to the one presented in \cite{stochastic_FW_2014}. However, in contrast to the above works, we do not assume any structure in the optimization problem or in the sampling. A detailed proof can be found in the Appendix.
As in the deterministic case, Proposition \ref{prop2} implies a complexity bound of $\cO(1/\varepsilon)$ iterations to reach an approximate solution $\alpha^{(k)}$ such that $\mathbb{E}_{\cS^{(k)}}[f(\alpha^{(k)})] - f(\alpha^{*}) \leq \varepsilon$.

\subsection{Choosing the Sampling Size}\label{sampling_sec}
When using a randomized FW iteration it is important to choose the sampling size in a sensible way. Indeed, some recent works showed how this choice entails a tradeoff between accuracy (in the sense of premature stopping) and complexity, and henceforth CPU time \cite{NIPS14}.
This kind of approximation is motivated by the following result, which suggests that it is reasonable to 
pick $|\mathcal{S}| \ll p$.

\vspace{0.1cm}
\begin{theorem}[\cite{Smola01Learning}, Theorem 6.33]\label{sampling_thm}
Let $\mathcal{D} \subset \mathbb{R}$ s.t. $|\mathcal{D}| = p$ 
and let $\mathcal{D}^{\prime} \subset \mathcal{D}$ be a random subset
of size $\samplesize$. Then, the probability that the largest element in $\mathcal{D}^{\prime}$ is greater than or equal to $\tilde p$ elements of $\mathcal{D}$ is at least $1-(\frac{\tilde p}{p})^{\samplesize}$.
\end{theorem}

The value of this result lies in the ability to obtain probabilistic bounds for the quality of the sampling \emph{independently of the problem size $p$}.
For example, in the case of the Lasso problem, where $\mathcal{D} = \{|\nabla f(\alpha^{(k)})_1|,\ldots,|\nabla f(\alpha^{(k)})_p|\}$ and $\mathcal{D}^{\prime} = \{ |\nabla f(\alpha^{(k)})_i| \mbox{ s.t. } i \in \mathcal{S} \}$, it is easy to see that it suffices to take $|\mathcal{S}| \approx 194$ to guarantee that, with probability at least $0.98$, $|\nabla f(\alpha^{(k)})_{i^{(k)}_\mathcal{S}}|$ lies between the $2\%$ largest gradient components (in absolute value), independently of $p$. This kind of sampling has been discussed for example in \cite{PARTAN_paper}. 

The issue with this strategy is that, for problems with very sparse solutions (which is the case for strong levels of regularization), even a large confidence interval does not guarantee that the algorithm can sample a good vertex in most of the iterations. Intuitively, the sampling strategy should allow the algorithm to detect the set of vertices active at the optimum, which correspond, at various stages of the optimization process, to descent directions for the objective function. In sparse approximation problems, extracting a sampling set without relevant features may imply adding \quotes{spurious} components to the solution, reducing the sparseness of the model we want to find. 

A more effective strategy in this context would be ask for a certain probability that the sampling will include at least one of the \quotes{optimal} features. Letting $\activeset$ be the index set of the active vertices at the optimum, and denoting $\activesize=|\activeset|$ and $\samplesize=|\randomset|$, we have
\begin{equation}\label{sampling_ineq}
P(\activeset \cap \randomset = \emptyset) =  \prod_{j=0}^{\samplesize-1} \left(1-\frac{\activesize}{p-j}\right) \leq  \left(1-\frac{\activesize}{p}\right)^{\samplesize},
\end{equation}
with the latter inequality being a reasonable approximation if $\samplesize \ll p$. From (\ref{sampling_ineq}), we can guarantee that $\activeset \cap \randomset \neq \emptyset$ with probability at least $\rho$ by imposing:
\begin{equation}\label{sampling_rule}
\left(1-\frac{\activesize}{p}\right)^{\samplesize} \leq  (1 - \rho)  \ \Leftrightarrow  \ \samplesize \geq \frac{\ln(1 - \rho)}{ \ln\left(1-\frac{\activesize}{p}\right)} = \frac{\ln(\mbox{\small confidence})}{\ln(\mbox{\small sparseness})} \,.
\end{equation}

On the practical side, this sampling strategy often implies taking a larger $\kappa$. Assuming that the fraction of relevant features ($s/p$) is constant, we essentially get the bound for $\kappa$ provided by Theorem \ref{sampling_thm}, which is independent of $p$. However, for the worst case $s/p \to 0$, we get
\begin{equation}\label{sampling_rule}
\frac{\ln(1 - \rho)}{ \ln\left(1-\frac{\activesize}{p}\right)} \approx \left(\frac{-\ln(1 - \rho)}{s}\right) \ p \, ,
\end{equation}
which suggests to use a sampling size proportional to $p$. A more involved strategy, which exploits the incremental structure of the FW algorithm, is using a large $\kappa$ at early iterations and smaller values of $\kappa$ as the solution gets more dense. If the optimal solution is very sparse, the algorithm requires few expensive iterations to converge. In contrast, if the solution is dense, the algorithm requires more, but faster, iterations (e.g. for a confidence $1-\rho = 0.98$ and $s/p=0.02$ the already mentioned $\kappa=194$ suffices). For the problems considered in this paper, setting $\kappa$ to a small fraction of $p$ works well in practice, as shown by the experiments in the next Section.

%% file: lemma_1_proof.tex

\begin{lemma}\label{tech_lemma}
   Let $\mathcal{S}$ be picked at random from the set of all equiprobable $\samplesize$-subsets of $\{1,\ldots,p\}$, $1\leq \samplesize
   \leq p$, and let $v$ be any vector in $\mathbb{R}^p$.
Then
\begin{equation*}
  \mathbb{E}\left[\Bigl(\sum_{i\in \mathcal{S}} \mathbf{e}_i \mathbf{e}_i^T \Bigr) v\right] = \frac{\samplesize}{p}\, v\,.
\end{equation*}
\begin{proof}
  Let
  $A_\mathcal{S} = \sum_{i\in \mathcal{S}} \mathbf{e}_i \mathbf{e}_i^T$ and
  $\left(A_\mathcal{S}\right)_{ij}$ an element of $A_\mathcal{S}$. For $i\neq j$,
  $\left(A_\mathcal{S}\right)_{ij} = 0$ and
  $\mathbb{E}\,\bigl[\left(A_\mathcal{S}\right)_{ij}\bigr] = 0.$ For $i= j$,
  $\left(A_\mathcal{S}\right)_{ij}$ is a Bernoulli random variable with
  expectation $\tfrac{\samplesize}{p}$.
  In fact, $\left(A_\mathcal{S}\right)_{ii} = 1$ iff $i\in \mathcal{S}$;
  as there are $\binom{p-1}{\samplesize-1}$ $\samplesize$-subsets of  $\{1,\ldots,p\}$
  containing $i$,
  \[
    \mathbb{P}(i\in \mathcal{S}) = \binom{p-1}{\samplesize-1}\binom{p}{\samplesize}^{-1} = \samplesize/p = \mathbb{P}\bigl(\left(A_\mathcal{S}\right)_{ii} =1\bigr) = \mathbb{E}\,\bigl[\left(A_\mathcal{S}\right)_{ii}\bigr].
    \]
  Therefore, for $i\in\{1,\ldots,p\}$,
\[
    \mathbb{E}\,\bigl[\left(A_\mathcal{S}\right)_{i*} v\bigr] =
    \mathbb{E}\,\Bigl[\sum_{j=1}^p\left(A_\mathcal{S}\right)_{ij} v_j \Bigr]  = \sum_{j=1}^pv_j\,\mathbb{E}\, \bigl[\left(A_\mathcal{S}\right)_{ij} \bigr]= \frac{\samplesize}{p} \,v_i\,.
\]
\end{proof}
\end{lemma}

%% file: LASSO_paper_Sec5_Experiments.tex

\section{Numerical Experiments}\label{experiments}
In this section, we assess the practical effectiveness of the randomized FW algorithm by performing experiments on both synthetic datasets and real-world benchmark problems with hundreds of thousands or even millions of features.  
The characteristics of the datasets are summarized in Table \ref{TabDatasets}, where we denote by $m$ the number of training examples, by $t$ the number of test examples, and by $p$ the number of features. The synthetic datasets were generated with the Scikit-learn function \texttt{make\_regression} \cite{scikit-learn}. Each of them comes in two 
versions corresponding to a different number of relevant features in the true linear model used to generate the data (32 and 100 features for the problem of size $p = 10000$, and 158 and 500 features for that of size $p = 50000$).
The real large-scale regression problems \textbf{E2006-tfidf} and \textbf{E2006-log1p}, and the datasets \textbf{Pyrim} and \textbf{Triazines}, are available from \cite{SVMLIB}. 


\begin{table}
\centering
\begin{tabular}{@{}lrrr}
\hline
{Dataset} \T& ${m}$ & $t$ & $p$  \B\\
\hline
\T {\textbf{Synthetic-10000}}     & $200$ & $200$   & $10,000$   \\
{\textbf{Synthetic-50000}}  & $200$  & $200$   & $50,000$     \\
{\textbf{Pyrim}}   & $74$ & $--$   & $201,376$  \\
{\textbf{Triazines}}   & $186$ & $--$   & $635,376$  \\
{\textbf{E2006-tfidf}}   & $16,087$ & $3,308$   & $150,360$  \\
{\textbf{E2006-log1p}}\B    & $16,087$  & $3,308$ & $4,272,227$ \\
\hline
\end{tabular}
\caption{\small \label{TabDatasets} List of the benchmark datasets used in the experiments.}
\end{table}


In assessing the performance of our method, we used as a baseline the following algorithms, which in our view, and according to the properties summarized in Table \ref{methods_comparison}, can be considered among the most competitive solvers for Lasso problems:
\begin{itemize}
\item The well-known CD algorithm by Friedman \textit{et al.}, as implemented in the Glmnet package \cite{friedman2010}. This method is highly optimized for the Lasso and is considered a state of the art solver in this context.
\item The SCD algorithm as described in \cite{shwartz2011}, which is significant both for being a stochastic method and for having better theoretical guarantees than the standard cyclic CD.
\item The Accelerated Gradient Descent with projections for both the regularized and the constrained Lasso, as this algorithm guarantees an optimal complexity bound. We choose as a reference the implementation in the SLEP package by Liu \textit{et al.} \cite{SLEP}.
\end{itemize}
Among other possible first-order methods, the classical SGD suffers from a worse convergence rate, and its variant SMIDAS has a complexity bound which depends on $p$, thus we did not include them in our experiments. Indeed, the authors of \cite{shwartz2010} conclude that SCD is both faster and produces significantly sparser models compared to SGD. Finally, although the recent GeoLasso algorithm of \cite{geolasso} is interesting because of its close relationship with FW, its running time and memory requirements are clearly higher compared to the above solvers.

\begin{table}[ht!]\label{methods_comparison}
\centering
\begin{footnotesize}
\begin{tabular}{|ccccc|}
\hline
\T \textbf{Approach} & \textbf{Form} & \textbf{Number of} & \textbf{Complexity} & \textbf{Sparse}\\
\B & & \textbf{Iterations} & \textbf{per Iteration} & \textbf{Its.}\\
\hline
\T Accelerated Gradient + Proj. & (\ref{eq:LASSO}) & $\cO(1/\sqrt{\epsilon})$ &  $\cO(mp+p)$${\tiny\dagger_1}$ & No\\
\B \cite{liu2009efficient} & & &  &\\  
\hline
\T Accelerated Gradient + Reg. Proj.  & (\ref{eq:LASSO-unc}) & $\cO(1/\sqrt{\epsilon})$ &  $\cO(mp+p)$$\scriptsize{\dagger_1}$ & No\\
\B \cite{liu2010efficient} & & &  &\\  
\hline
\T Cyclic Coordinate Descent & (\ref{eq:LASSO-unc}) & Unknown & $\cO(mp)$$\dagger_2$ & Yes\\
\B($\mbox{CD}$) \cite{friedman2007,friedman2010} & & & &\\  
\hline
\T Stochastic Gradient Descent & (\ref{eq:LASSO-unc}) & $\cO(1/\epsilon^2)$ & $\cO(p)$ & No\\
\B ($\mbox{SGD}$) \cite{langford2009sparse} & & &  &\\  
\hline
\T Stochastic Mirror Descent & (\ref{eq:LASSO-unc}) & $\cO(\log(p)/\epsilon^2)$ & $\cO(p)$ & No\\
\B \cite{shwartz2011} & & &  &\\  
\hline
\T GeoLasso \cite{geolasso} & (\ref{eq:LASSO}) & $\cO(1/\epsilon)$ & $\cO(mp + a^2)$ & Yes \B\\
\hline
\T Frank-Wolfe ($\mbox{FW}$)  \cite{Jaggi2013ICMLa} & (\ref{eq:LASSO}) & $\cO(1/\epsilon)$ & $\cO(mp)$ & Yes \B\\
\hline
\T Stochastic Coord. Descent ($\mbox{SCD}$) & (\ref{eq:LASSO-unc}) & $\cO(p/\epsilon)$ & $\cO(m)$$\dagger_3$ & Yes\\
\B  \cite{richtarik2014} & & &  &\\  
\hline
\T Stochastic Frank-Wolfe & (\ref{eq:LASSO}) & $\cO(1/\epsilon)$ & $\cO(m|\cS|)$  & Yes \B\\
\hline
\end{tabular}
\end{footnotesize}
\caption{\label{methods_comparison} Methods proposed for scaling the Lasso and their complexities. Here, $a$ denotes the number of active features at a given iteration, which in the worst case is $a=\mbox{rank}(X)\leq \mbox{min}(m,p)$. A method is said to have \emph{sparse iterations} if a non trivial bound for the number of non-zero entries of each iterate holds at any moment. $\dagger_1$ $\cO(p)$ is required for the projections. $\dagger_2$ An iteration of cyclic coordinate descent corresponds to a complete cycle through the features. $\dagger_3$ An iteration of SCD corresponds to the optimization on a single feature. }
\end{table}

Since an appropriate level of regularization needs to be automatically selected in practice, the algorithms are compared by computing the entire regularization path on a range of values of the regularization parameters $\lambda$ and $\delta$ (depending on whether the method solves the penalized or the constrained formulation). 
Specifically, we first estimate two intervals $[\lambda_{\min},\lambda_{\max}]$ and $[\delta_{\min},\delta_{\max}]$, and then solve problems (\ref{eq:LASSO-unc}) and (\ref{eq:LASSO}) on a $100$-point parameter grid in logarithmic scale.
For the penalized Lasso, we use $\lambda_{\min} = \lambda_{\max}/100$, where $\lambda_{\max}$ is determined as in the Glmnet code. 
Then, to make the comparison fair (i.e. to ensure that all the methods solve the same problems according to the equivalence in Section \ref{lasso-sec}), we choose for the constrained Lasso $\delta_{\max} = \|\alpha_{\min}\|_1$ and  $ \delta_{\min} = \delta_{\max}/100$, where $\alpha_{\min}$ is the solution obtained by Glmnet with the smallest regularization parameter $\lambda_{\min}$ and a high precision ($\varepsilon = 10^{-8}$). The idea is to give the solvers the same \quotes{sparsity budget} to find a a solution of the regression problem.

\subsubsection*{Warm-start strategy} As usually done in these cases, and for all the methods, we compute the path using a warm-start strategy where each solution is used as an initial guess for the optimization with the next value of the parameter. Note that we always start from the most sparse solution. This means that in the cases of CD, SCD and regularized SLEP we start from $\lambda_{\max}$ towards $\lambda_{\min}$, while for FW and constrained SLEP we go from $\delta_{\min}$ towards $\delta_{\max}$. Furthermore, since $\delta < \| \alpha^R \|_1$, where $\alpha^R$ is the unconstrained solution, we know that the solution will lie on the boundary, therefore we adopt a heuristic strategy whereby the previous solution is scaled so that its $\ell_1$-norm is $\delta$. Both algorithms are initialized with the null solution as the initial guess. 
Regarding the stopping criterion for each problem in the 
path, we stop the current run when $\| \alpha^{(k+1)}-\alpha^{(k)}\|_{\infty} \leq \varepsilon$ for all the algorithms.
Other choices are possible (for example, FW methods are usually stopped using a duality gap-based criterion \cite{Jaggi2013ICMLa}), but this is the condition used by default to stop CD in Glmnet. A value of $\varepsilon = 10^{-3}$ is used in the following experiments. 
All the considered algorithms have been coded in C++. The code and datasets used in this paper are available from public repositories on Github (\url{https://github.com/efrandi/FW-Lasso}) and Dataverse (\url{https://goo.gl/PTQ05R}), respectively. The SLEP package has a Matlab interface, but the key functions are all coded in C. Overall, we believe our comparison can be considered very fair. 
We executed the experiments on a 3.40GHz Intel i7 machine with 16GB of main memory running CentOS. For the randomized experiments, results were averaged over 10 runs.

\subsection{\quotes{Sanity Check} on the Synthetic Datasets}
The aim of these experiments is not to measure the performance of the algorithms (which will be assessed below on four larger, real-life datasets), but rather to compare their ability to capture the evolution of the most relevant features of the models, and discuss how this relates to their behaviour from an optimization point of view.
To do this, we monitor the values of the 10 most relevant features along the path, as computed by both the CD and FW algorithms, and plot the results in Figures \ref{variables_10000} and \ref{variables_50000}.
To determine the relevant variables, we use as a reference the regularization path obtained by Glmnet with $\varepsilon = 10^{-8}$ (which is assumed to be a reasonable approximation of the exact solution), and compute the 10 variables having, on average, the highest absolute value along the path. As this experiment is intended mainly as a sanity check to verify that our solver reconstructs the solution correctly, we do not include other algorithms in the comparison.
In order to implement the random sampling strategy in the FW algorithm, 
we first calculated the average number of active features along the path, rounded up to the nearest integer, as an empirical estimate of the sparsity level. 
Then, we chose $| \cS |$ based on the probability $\rho$ of capturing at least one of the relevant features at each sampling, according to (\ref{sampling_rule}). A confidence level of $99\%$ was used for this experiment, leading to sampling sizes of 372 and 324 points for the two problems of size 10000, and of 1616 and 1572 points for those of size 50000.

\begin{figure}[h!!!]
\begin{center}
\subfigure[]{
   {\includegraphics[scale = 0.09]{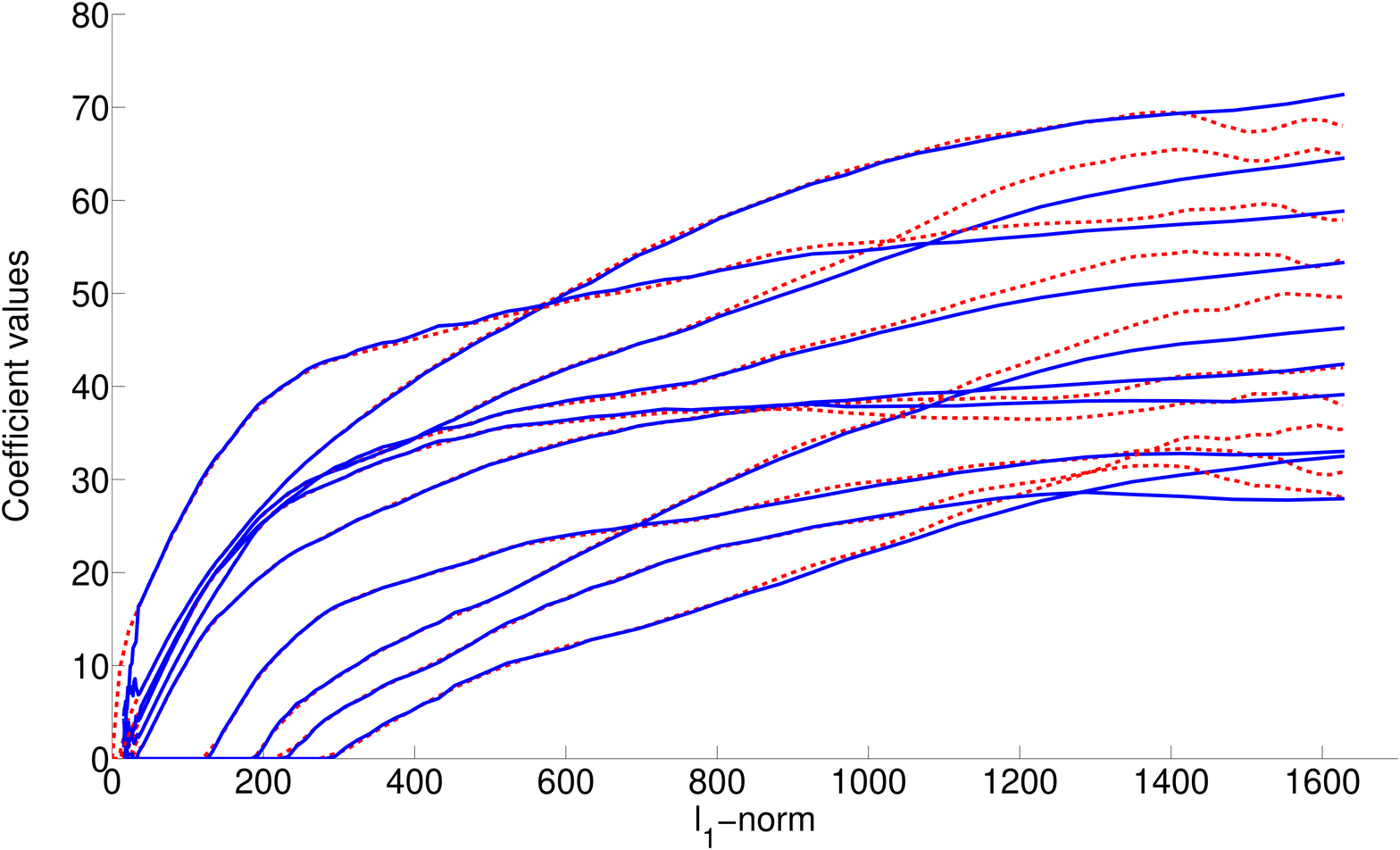}}}
\subfigure[]{
   {\includegraphics[scale = 0.09]{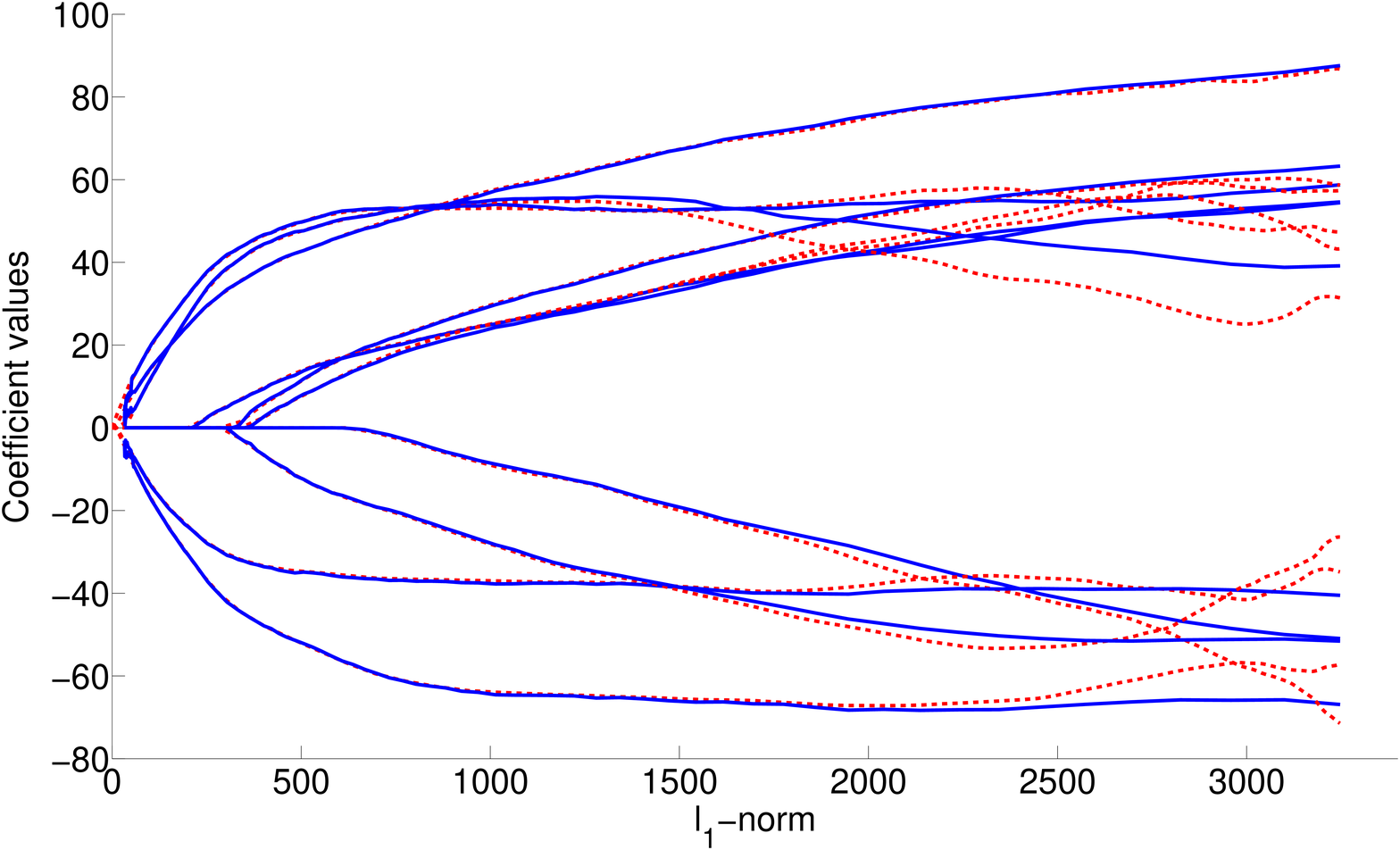}}}
\caption{\label{variables_10000} Growth of the 10 most significant features on the synthetic problem of size {10000}, with 32 (a) and 100 (b) relevant features. Results for CD are in red dashed lines, and in blue continuous lines for FW.} 
\end{center}
\end{figure} 

\begin{figure}[h!!!]
\begin{center}
\subfigure[]{
   {\includegraphics[scale = 0.09]{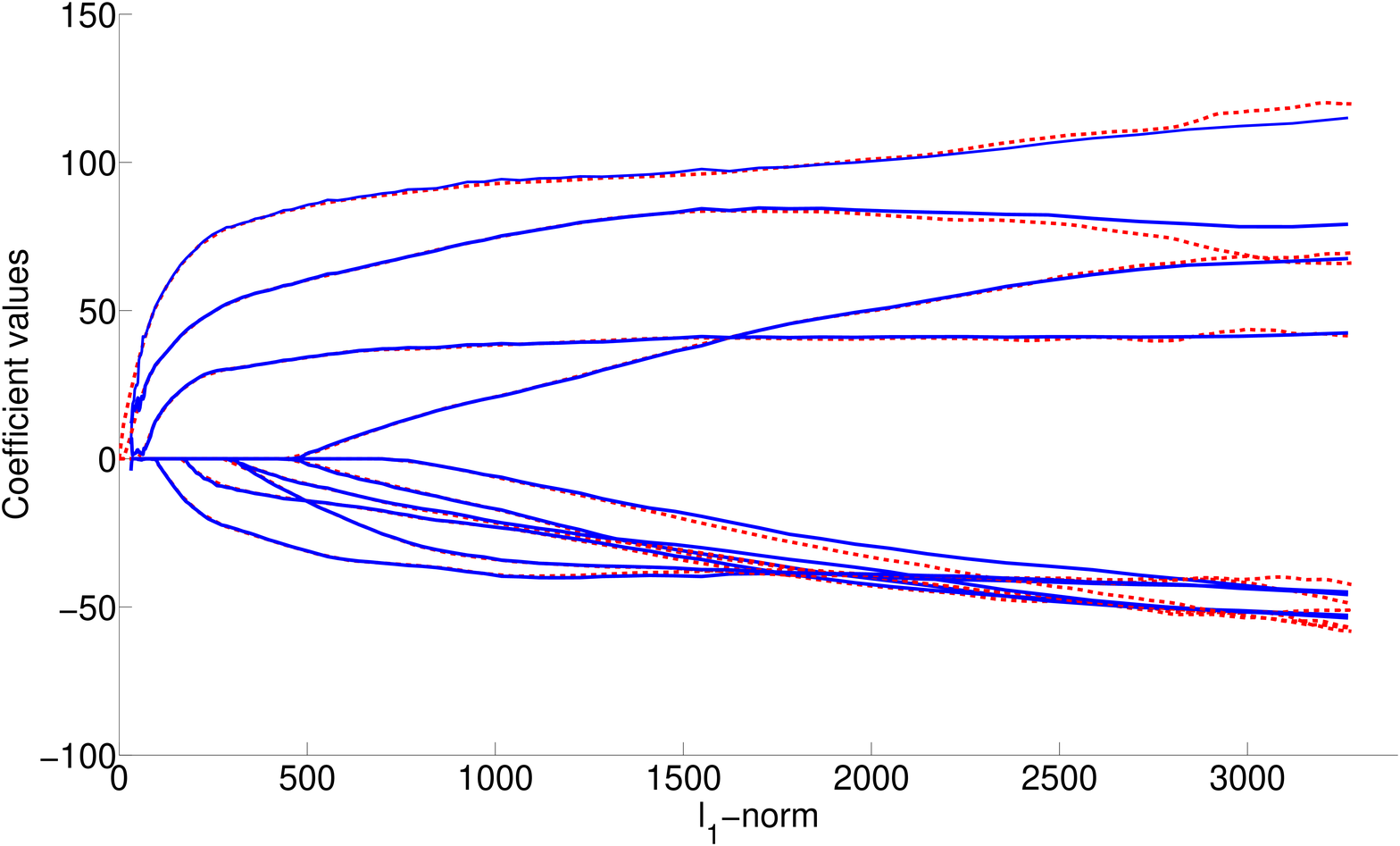}}}
\subfigure[]{
   {\includegraphics[scale = 0.09]{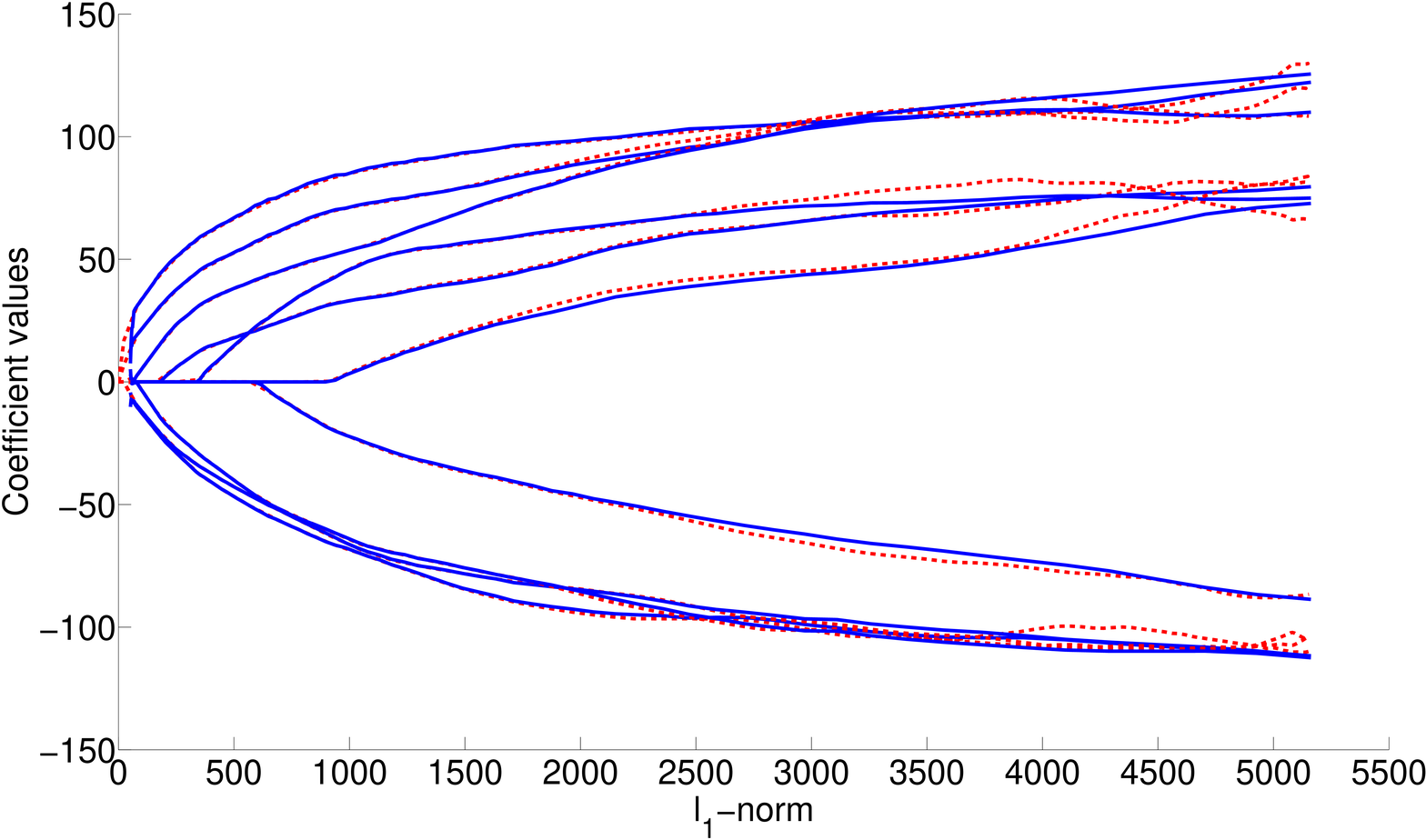}}}
\caption{\label{variables_50000} Growth of the 10 most significant features on the synthetic problems of size {50000}, with 158 (a) and 500 (b) relevant features. Results for CD are in red dashed lines, and in blue continuous lines for FW.} 
\end{center}
\end{figure} 


Figure \ref{synthetic_errors} depicts, for two of the considered problems, the prediction error on the test set along the path found by both algorithms. 
It can be seen that both methods are able to find the same value of the best prediction error (i.e. to correctly identify the best model). The FW algorithm also seems to obtain slightly better results towards the end of the path, consistently with the fact that the coefficients of the most relevant variables tend to be more stable, compared with CD, when the regularization is weaker.

\begin{figure}
\begin{center}
\subfigure[]{
{\includegraphics[scale = 0.09]{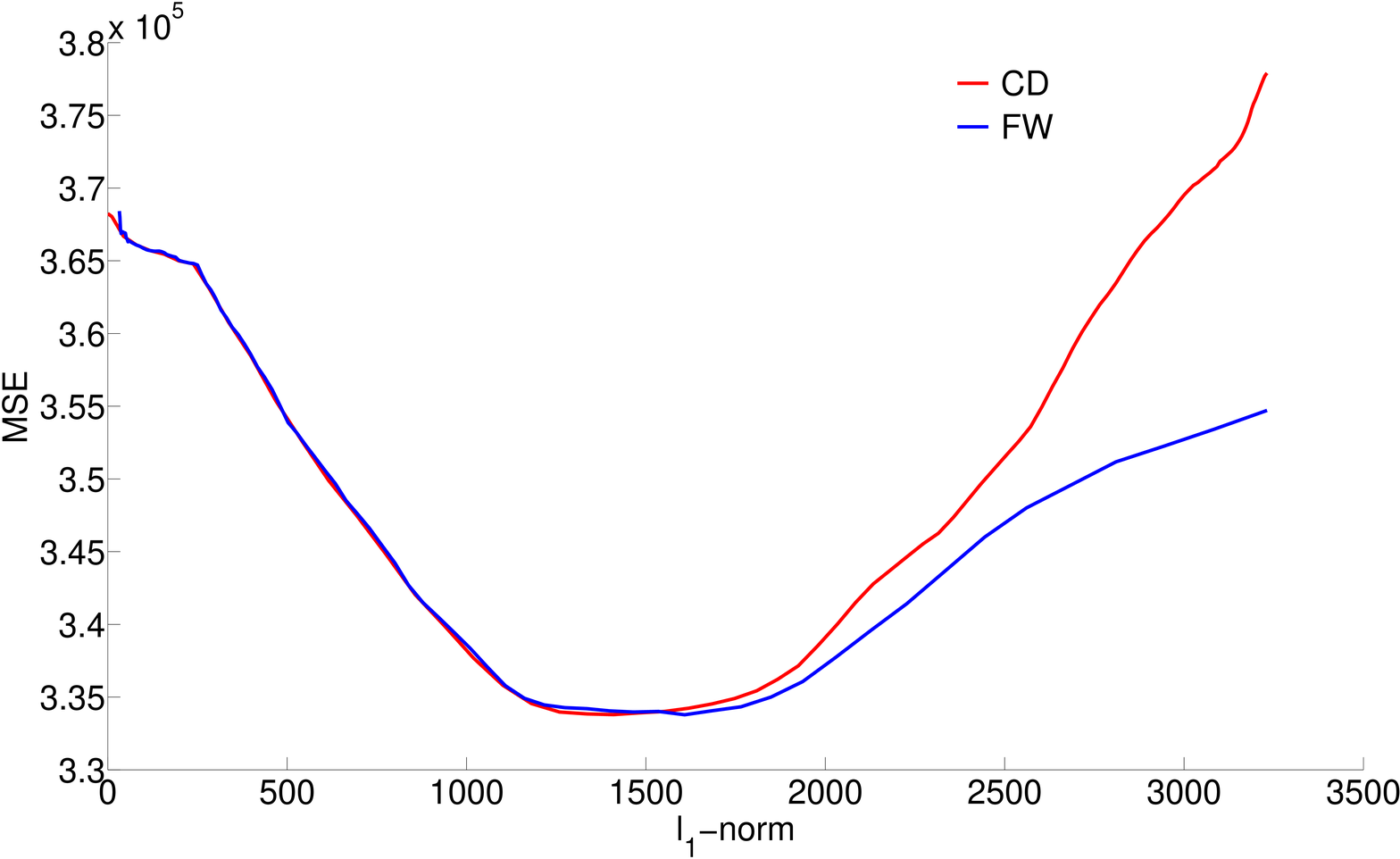}}}
\subfigure[]{
   {\includegraphics[scale = 0.09]{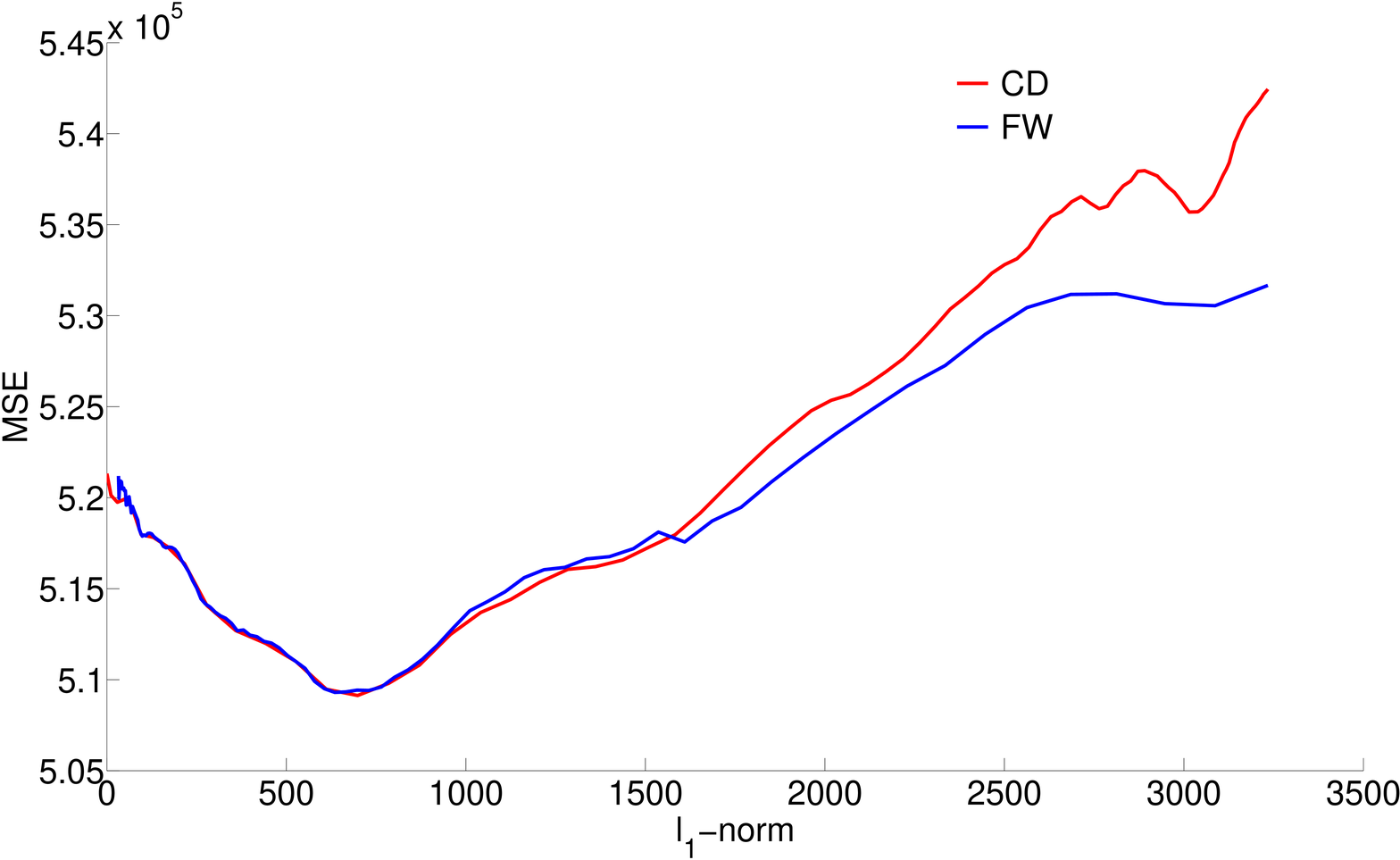}}}
\caption{\label{synthetic_errors} Test error ($\ell_1$-norm vs. MSE) for problems \textbf{Synthetic-10000} (100 relevant features) \textbf{Synthetic-50000} (158 relevant features). Results for CD are in red, and in blue for FW.} 
\end{center}
\end{figure} 


\subsection{Results on Large-scale Datasets}
In this section, we report the results on the problems \textbf{Pyrim}, \textbf{Triazines}, \textbf{E2006-tfidf} and \textbf{E2006-log1p}.
These datasets represent actual large-scale problems of practical interest. 
The \textbf{Pyrim} and \textbf{Triazines} datasets stem from two quantitative structure-activity relationship models (QSAR) problems, where the biological responses of a set of molecules are measured and statistically related to the molecular structure on their surface. We expanded the original feature set by means of product features, i.e. modeling the response variable $y$ as a linear combination of polynomial basis functions, as suggested in \cite{poly_features}. For this experiment, we used product features of order $5$ and $4$ respectively, which leads to large-scale regression problems with $p=201,376$ and $p=635,376$.
Problems \textbf{E2006-tfidf} and \textbf{E2006-log1p} stem instead from the real-life NLP task of predicting the risk of investment (i.e. the stock return volatility) in a company based on available financial reports \cite{E2006-paper}.

To implement the random sampling for the FW algorithm, we cannot use the technique described above for the experiments on the synthetic datasets, as in real problems we do not have an \textit{a-priori} estimate of the sparsity level. We therefore adopt a simpler strategy, and set $|\cS|$ to a fixed, small fraction of the total number of features. 
Our choices are summarized in Table \ref{TabSampling}.


\begin{table}[h!!]
\centering
\begin{tabular}{@{}lrrrr}
\hline
 \T  \% of $p$ & {\textbf{Pyrim}} & {\textbf{Triazines}} & {\textbf{E2006-tfidf}}  & {\textbf{E2006-log1p}} \B \\
\hline
\T {1\%} & 2,014 & 6,354 & 1,504 & 42,723 \\
{2\%} & 4,028 & 12,708 & 3,008 & 85,445 \\
{3\%}\B & 6,042 & 19,062 & 4,511 & 128,167 \\
\hline
\end{tabular}
\caption{\small \label{TabSampling} Sizes of the sampling set $|\cS|$ for the large-scale datasets.}
\end{table}

%
%


As a measure of the performance of the considered algorithms, we report the CPU time in seconds, the total number of iterations and the number of requested dot products (which account for most of the required running time for all the algorithms\footnote{Note that the SLEP code uses highly optimized libraries for matrix multiplication, therefore matrix-vector computations can be faster than naive C++ implementations.}) 
along the entire regularization path. Note that, in assessing the number of iterations, we consider one complete cycle of CD to be roughly equivalent to a full deterministic iteration of FW (since in both cases the complexity is determined by a full pass through all the coordinates) and to $p$ random coordinate explorations in SCD. 
Finally, in order to evaluate the sparsity of the solutions, we report the average number of active features along the path.
Results are displayed in Tables \ref{largescale_CD} and \ref{largescale_FW}. In the second table, the speed-ups with respect to the CD algorithm are also reported. It can be seen how for all the choices of the sampling size the FW algorithm allows for a substantial improvement in computational performance, as confirmed by both the CPU times and the machine-independent number of requested dot products (which are roughly proportional to the running times). The plain SCD algorithm performs somewhat worse than CD, something we attribute mainly to the fact that the Glmnet implementation of CD is a highly optimized one, using a number of \textit{ad-hoc} tricks tailored to the Lasso problem that we decided to preserve in our comparison. If we used a plain implementation of CD, we would expect to obtain results very close to those exhibited by SCD.

Furthermore, FW is always able to find the sparsest solution among the considered methods. The extremely large gap in average sparsity between FW and CD on one side, and the SLEP solvers on the other, is due to the fact that the latter compute in general dense iterates. Although the Accelerated Gradient Descent solver is fast and competitive from an optimization point of view, providing always the lower number of iterations as predicted by the theory, it is not able to keep the solutions sparse along the path. This behavior clearly highlights the advantage of using incremental approximations in the context of sparse recovery and feature selection. Importantly, note that the small number of features found by FW is not a result of the randomization technique: it is robust with respect to the sampling size, and additional experiments performed using a deterministic FW solver revealed that the average number of nonzero entries in the solution does not change even if the randomization is completely removed.

\begin{table}[h!!!]
\centering
\begin{footnotesize}
{\begin{tabular}{p{2.5cm}p{1.7cm}p{1.7cm}p{1.7cm}p{1.7cm}}
\hline
\T\B  & CD & SCD & SLEP Reg. & SLEP Const.  \\
\hline
\textbf{Pyrim} \T   & & & &  \\
	 Time (s)  & $6.22\e+00$  & $1.59\e+01$  & $5.43\e+00$ & $6.86\e+00$ \\
  	 Iterations &  $2.54\e+02$ & $1.44\e+02$ & $1.00\e+02$ & $1.12\e+02$ \\  
	 Dot products & $2.08\e+07$ & $2.90\e+07$ & $8.56\e+07$ & $1.29\e+08$  \\
  	 \B Active features & $68.4 $ & $ 116.6$ & $3,349$ & $13,030$  \\
\hline
\textbf{Triazines} \T   & & & &  \\
	 Time (s)  & $2.75\e+01$  & $8.42\e+01$  & $4.27\e+01$ & $5.93\e+01$ \\
  	 Iterations &  $2.62\e+02$ & $1.59\e+02$ & $1.01\e+02$ & $1.11\e+02$ \\  
	 Dot products & $6.80\e+07$ & $1.01\e+08$ & $2.87\e+08$ & $4.67\e+08$  \\
  	 \B Active features & $ 150.0$ & $330.8$ & 29,104 & 130,303  \\
\hline
\textbf{E2006-tfidf} \T   & & & &  \\
	 Time (s)  & $9.10\e+00$  & $2.19\e+01 $  & $1.24\e+01$ & $2.27\e+01$ \\
  	 Iterations &  $3.48\e+02$ & $2.01\e+02$ & $1.06\e+02$ & $2.50\e+02$ \\  
  	 Dot products & $2.04\e+07$ & $3.03\e+07 $ & $5.97\e+07$ & $1.37\e+08$  \\
	 \B Active features & $ 149.5 $ & $ 275.3 $ & $444.8$ & $724.3$  \\
\hline
\T \textbf{E2006-log1p}  & & & &  \\
	 Time (s)    &  $1.60\e+02$  & $4.92\e+02 $ & $1.00\e+02$ & $1.42\e+02$ \\
	 Iterations & $3.55\e+02$ & $1.99\e+02$ & $1.11\e+02$ & $1.18\e+02$ \\  
  	 Dot products &  $5.73\e+08$ & $8.50\e+08$ & $1.78\e+09$ & $2.85\e+09$  \\
	 \B Active features & $ 281.3 $ & $ 1158.2 $ & 12,806 & 54,704  \\
\hline
\end{tabular}}
\caption{\label{largescale_CD} Results for the baseline solvers on the large-scale problems \textbf{Pyrim}, \textbf{Triazines}, \textbf{E2006-tfidf} and \textbf{E2006-log1p}.}
\end{footnotesize}
\end{table}

\begin{table}[h!!!]
\centering
\begin{footnotesize}
{\begin{tabular}{p{2.5cm}p{1.7cm}p{1.7cm}p{1.7cm}}
\hline
\T\B  & FW $1\%$ & FW $2\%$ & FW $3\%$   \\
\hline
\textbf{Pyrim} \T   & & &   \\
	 Time (s)  & $2.28\e-01$  & $4.47\e-01$  & $6.60\e-01$   \\
	 Speed-up  &$27.3\times$  & $13.9\times$  & $9.4\times$  \\
  	 Iterations &  $2.77\e+02$ & $2.80\e+02$ & $2.77\e+02 $  \\  
  	 DotProd & $7.61\e+05 $ & $1.53\e+06$ & $2.28\e+06$  \\
	 \B Active features & $27.6$ & $28.1$ & $27.9$  \\
\hline
\textbf{Triazines} \T   & & &   \\
	 Time (s)  & $2.61\e+00$  & $5.31e+00$  & $8.19\e+00$   \\
	 Speed-up  &$10.5\times$  & $5.2\times$  & $ 3.4\times$  \\
  	 Iterations &  $7.15\e+02$ & $7.29\e+02$ & $7.43\e+02 $  \\  
  	 DotProd & $5.18\e+06 $ & $1.06\e+07$ & $1.61\e+07$  \\
	 \B Active features & $120.6$ & $117.5$ & $118.7$  \\
\hline
\textbf{E2006-tfidf} \T   & & &   \\
	 Time (s)  & $8.83\e-01$  & $1.76\e+00$  & $2.74\e+00$   \\
	 Speed-up  &$10.3\times$  & $5.2\times$  & $ 3.3\times$  \\
  	 Iterations &  $1.27\e+03$ & $1.35\e+03$ & $1.41\e+03 $  \\  
  	 DotProd & $1.97\e+06 $ & $4.35\e+06$ & $6.84\e+06$  \\
	 \B Active features & $123.7$ & $125.8$ & $127.1$  \\
\hline
\textbf{E2006-log1p} \T   & & &   \\
	 Time (s)  & $1.93\e+01$  & $4.14\e+01$  & $6.59\e+01$   \\
	 Speed-up  &$8.3\times$  & $3.9\times$  & $ 2.4\times$  \\
  	 Iterations &  $1.75\e+03$ & $1.91\e+03$ & $1.99\e+03 $  \\  
  	 DotProd & $7.90\e+07 $ & $1.71\e+08$ & $2.68\e+08$  \\
	 \B Active features & $196.9$ & $199.8$ & $203.7$  \\
\hline
\end{tabular}}
\caption{\label{largescale_FW} Performance metrics for stochastic FW on the large-scale problems \textbf{Pyrim}, \textbf{Triazines}, \textbf{E2006-tfidf} and \textbf{E2006-log1p}.}
\end{footnotesize}
\end{table}

To better assess the effect of using an incremental algorithm in obtaining a sparse model, we plot in Figure \ref{sparsity_plots} the evolution of the number of active features along the path on problems \textbf{E2006-tfidf} and \textbf{E2006-log1p}. It can be clearly seen how CD and FW (with the latter performing the best overall) are able to recover sparser solutions, and can do so without losing accuracy in the model, as we show below.


\begin{figure}[h!!!]
\begin{center}
\subfigure[]{
   {\includegraphics[scale = 0.09]{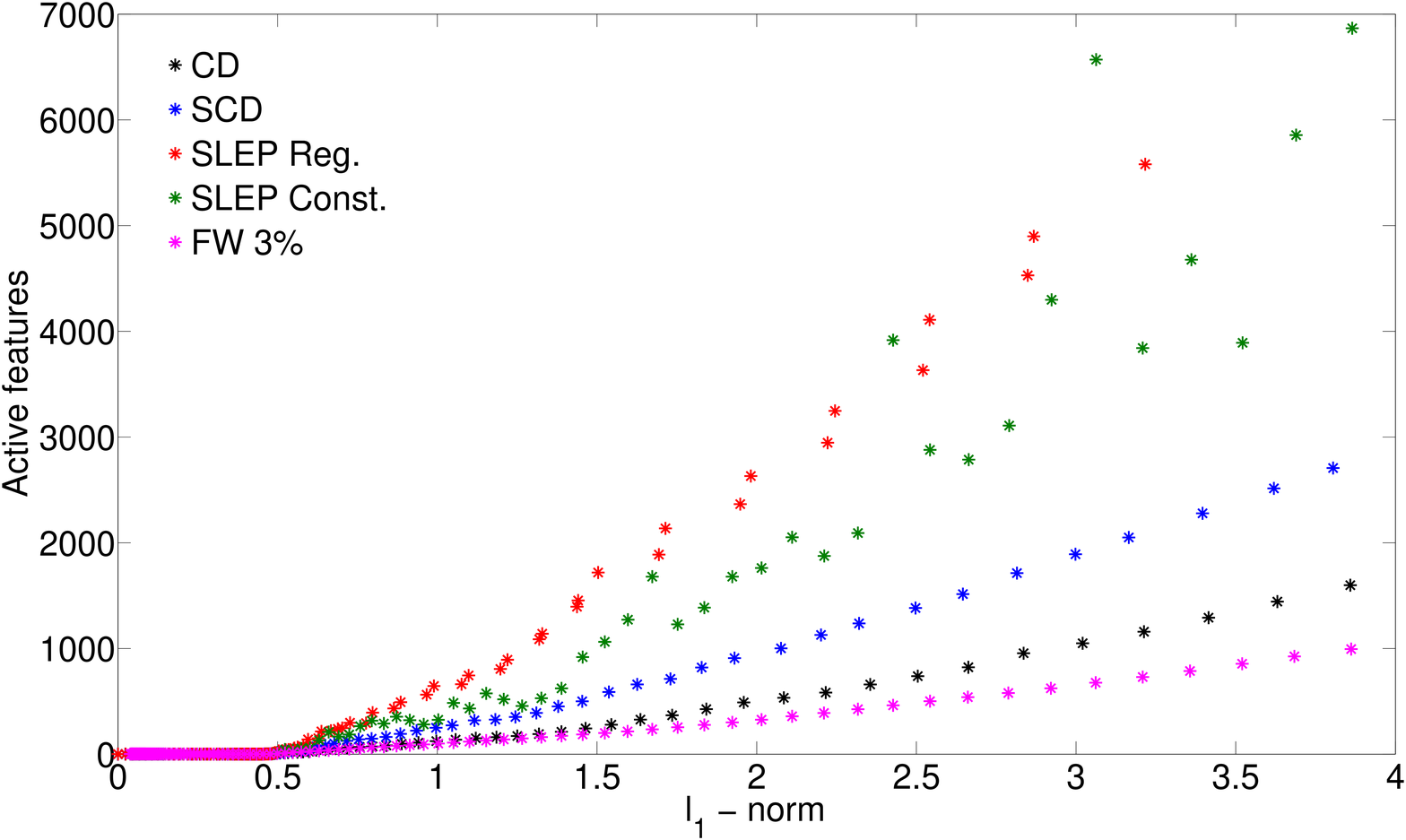}}}
\subfigure[]{
   {\includegraphics[scale = 0.09]{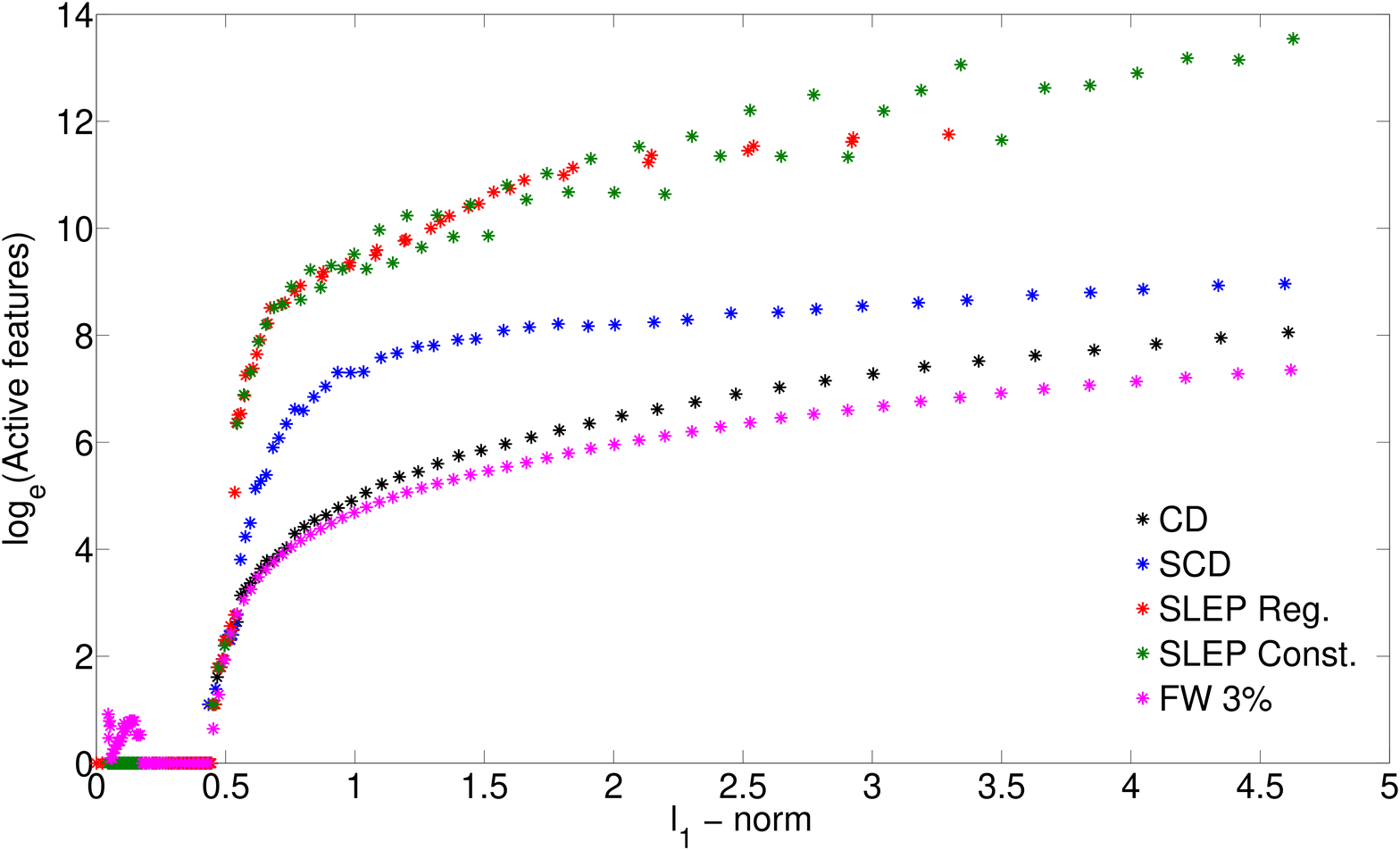}}}
\caption{\label{sparsity_plots} Sparsity patterns ($\ell_1$-norm vs. active coordinates) for problems \textbf{E2006-tfidf} (a) and \textbf{E2006-log1p} (b). The latter is plotted in a natural logarithmic scale due to the high number of features found by the SLEP solvers.} 
\end{center}
\end{figure} 

In order to evaluate the accuracy of the obtained models, we plot in Figures \ref{E2006_errors} and \ref{log1p_errors} the mean square error (MSE) against the $\ell_1$-norm of the solution along the regularization path, computed both on the original training set (curves \ref{E2006_errors}(a-c) and \ref{log1p_errors}(a-c)) and on the test set (curves \ref{E2006_errors}(b-d) and \ref{log1p_errors}(b-d)). Note that the value of the objective function in Problem (\ref{eq:LASSO}) coincides with the mean squared error (MSE) on the training set.
\begin{figure}
\begin{center}
\subfigure[]{
   {\includegraphics[scale = 0.1]{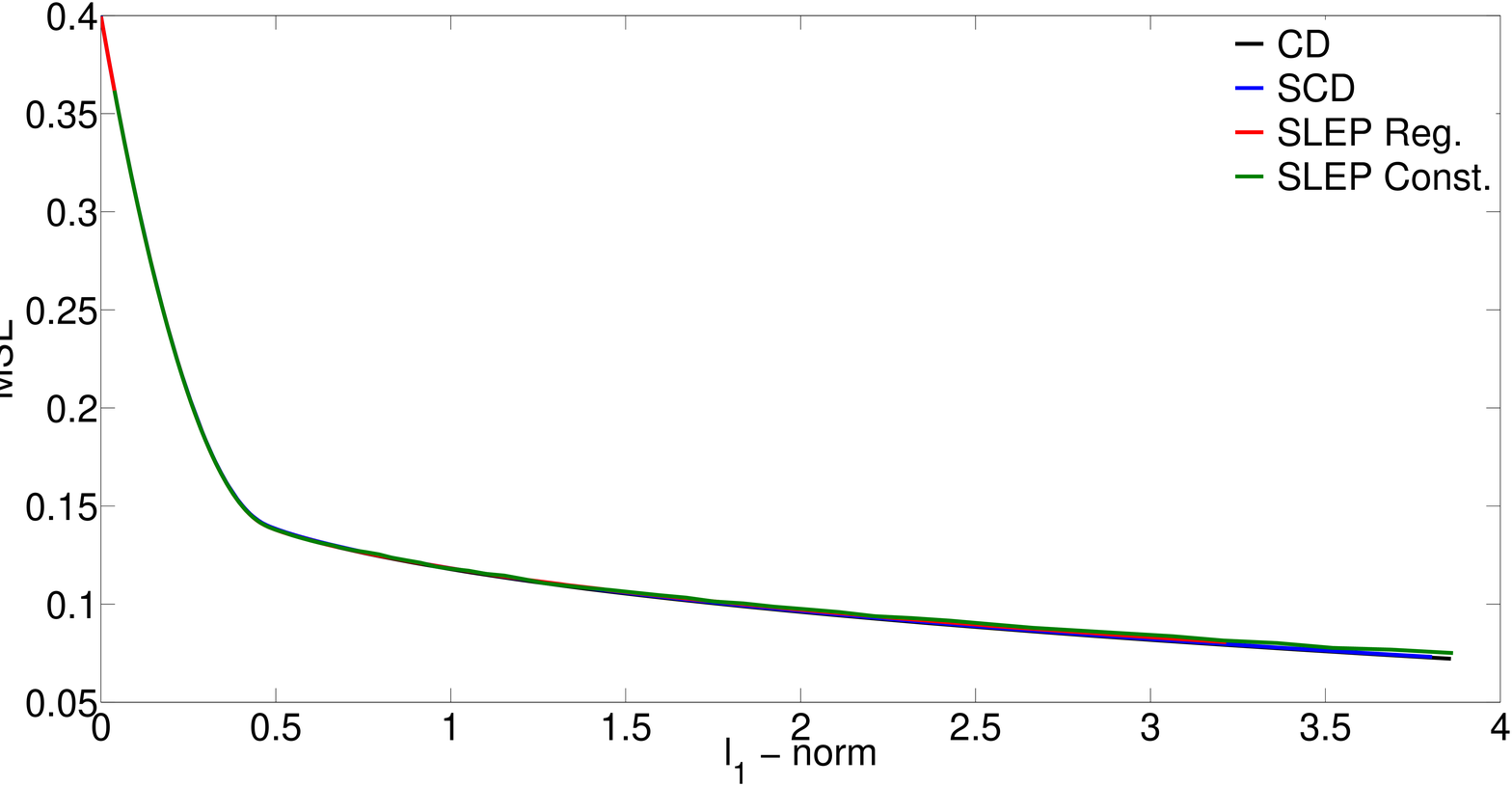}}}
\subfigure[]{
   {\includegraphics[scale = 0.1]{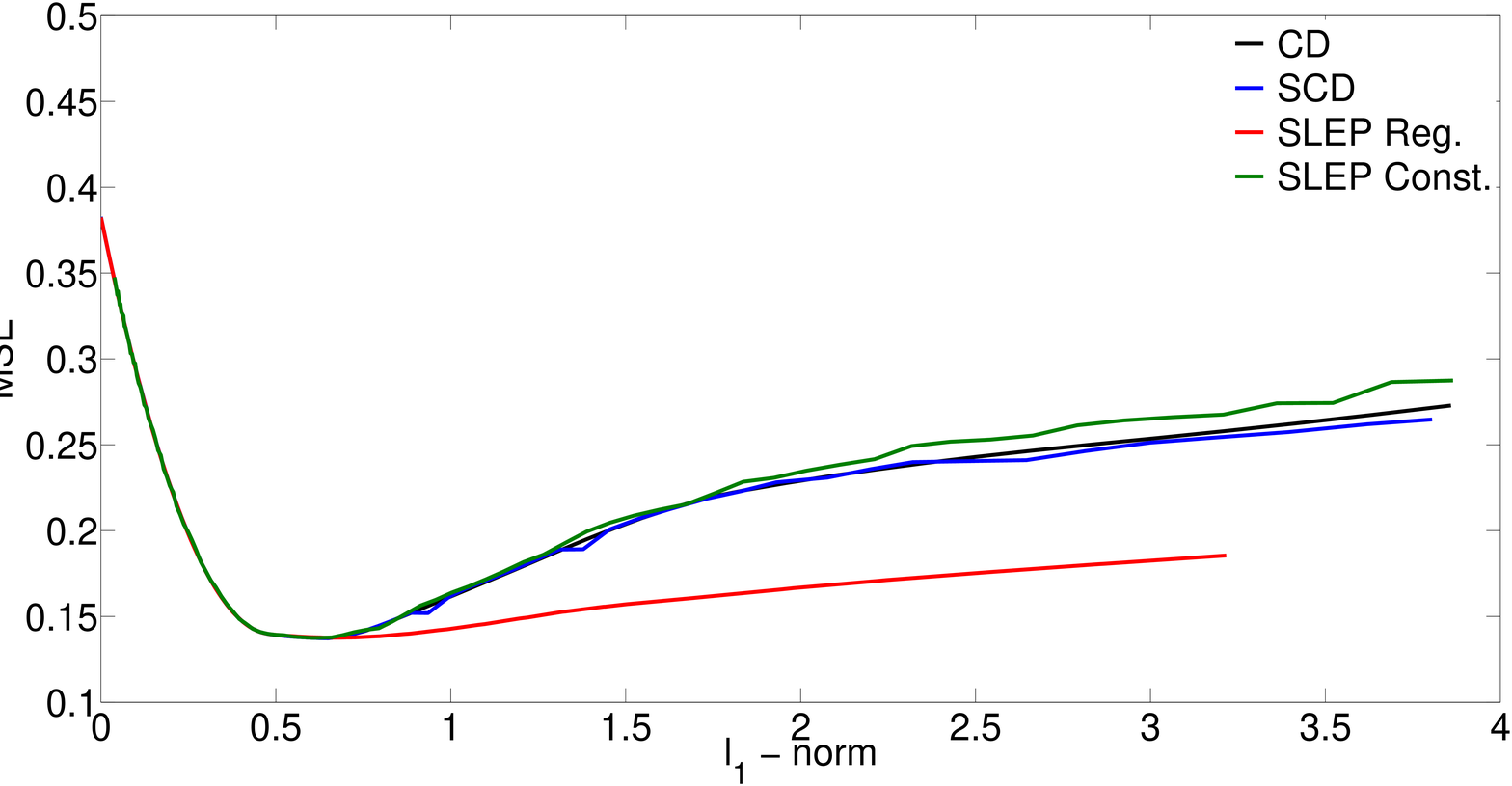}}}
\subfigure[]{
   {\includegraphics[scale = 0.1]{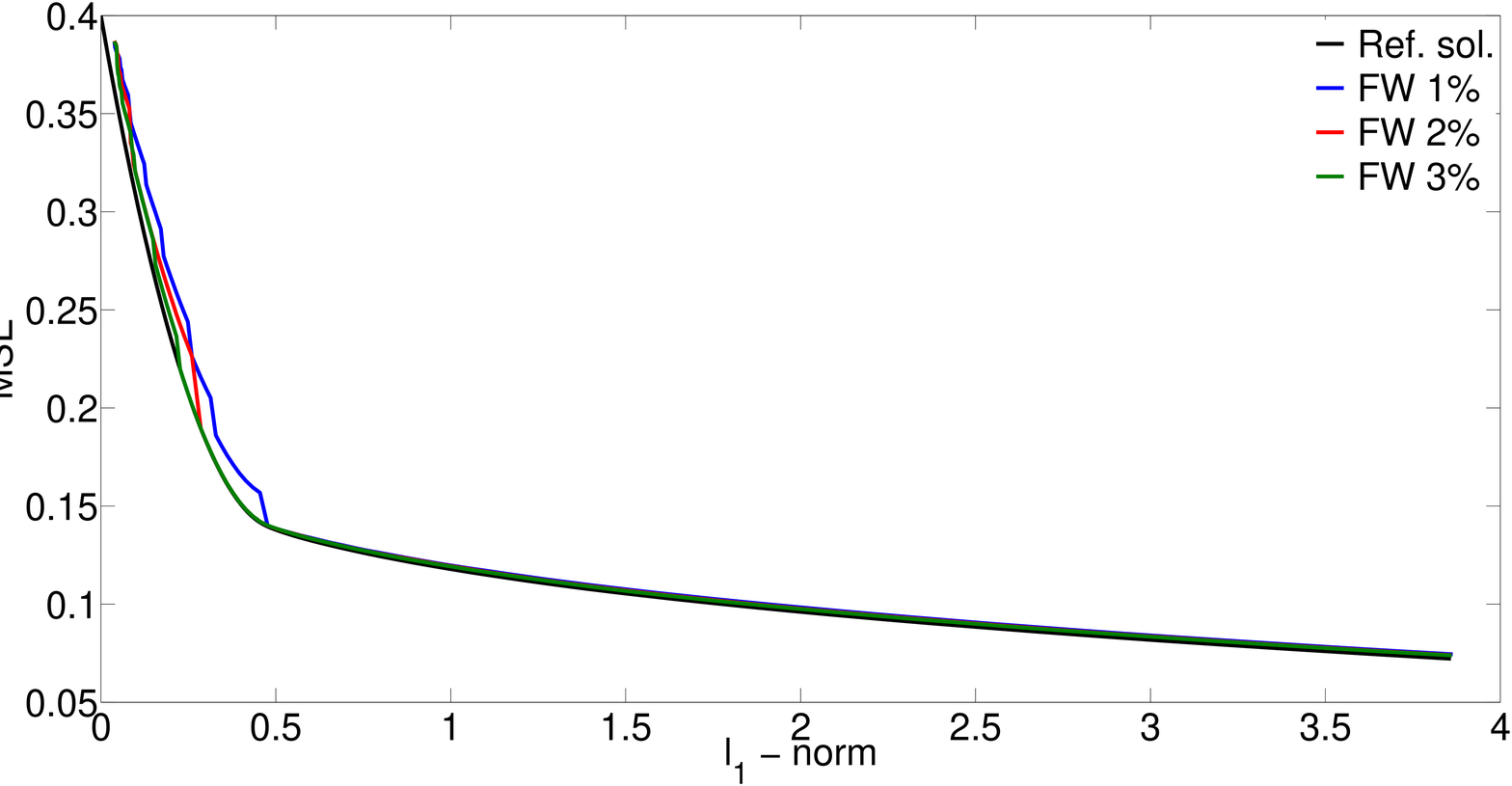}}}
\subfigure[]{
   {\includegraphics[scale = 0.1]{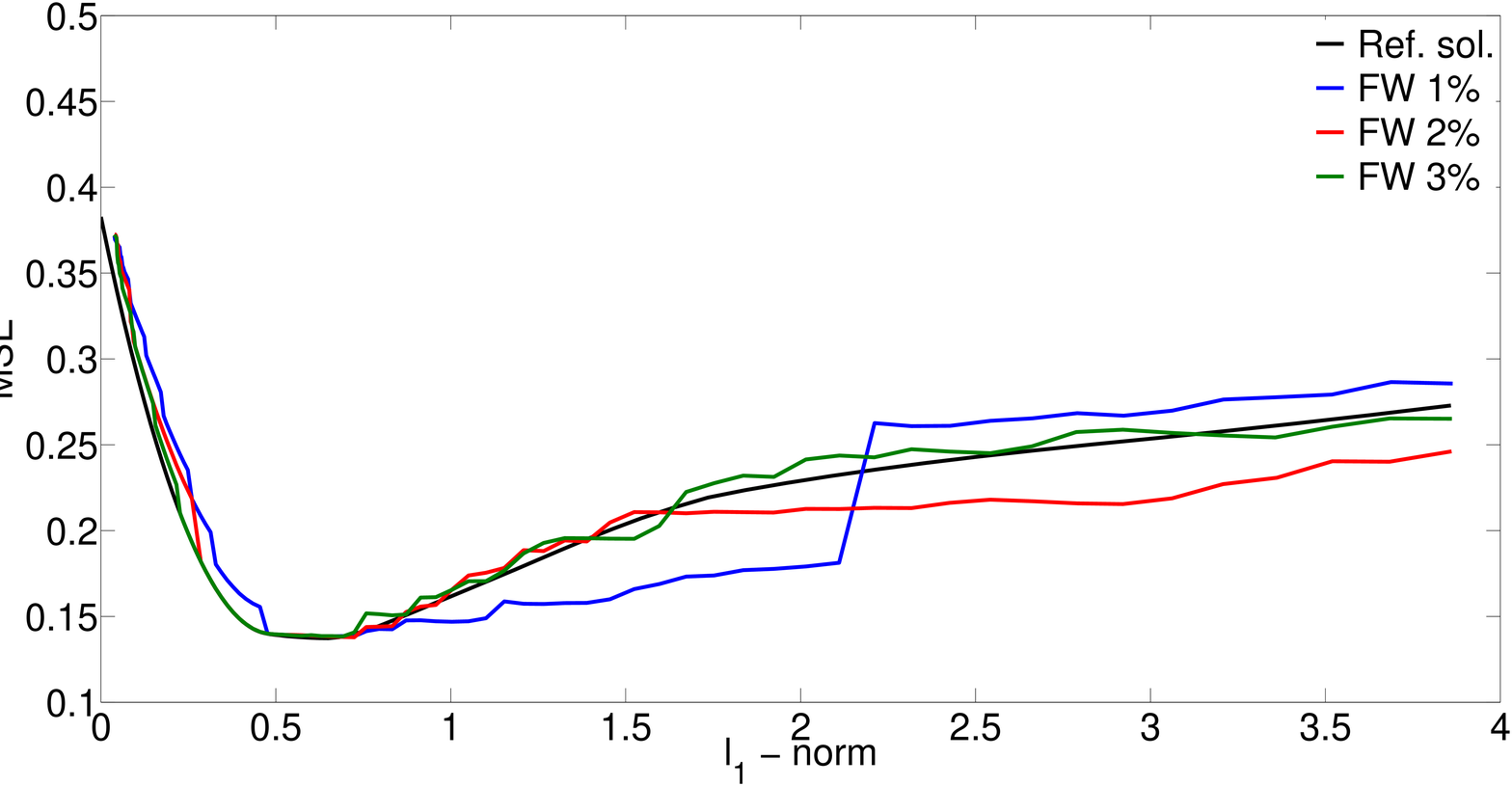}}}
\caption{\label{E2006_errors} Error curves ($\ell_1$-norm vs. MSE) for problem \textbf{E2006-tfidf}: on top, training error (a) and test error (b) for CD, SCD and SLEP; on bottom, training error (c) and test error (d) for stochastic FW.} 
\end{center}
\end{figure} 
\begin{figure}
\begin{center}
\subfigure[]{
   {\includegraphics[scale = 0.1]{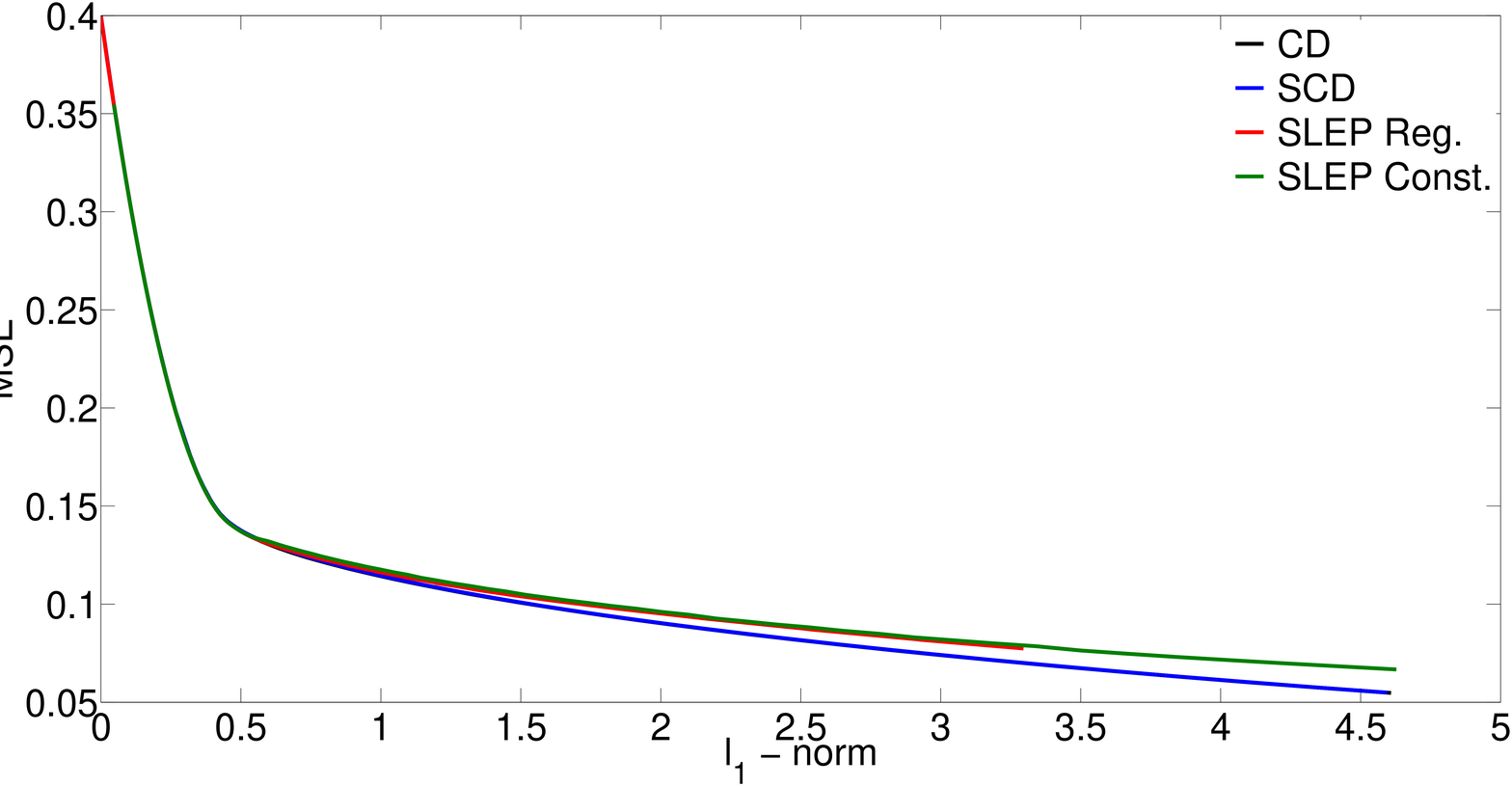}}}
\subfigure[]{
   {\includegraphics[scale = 0.1]{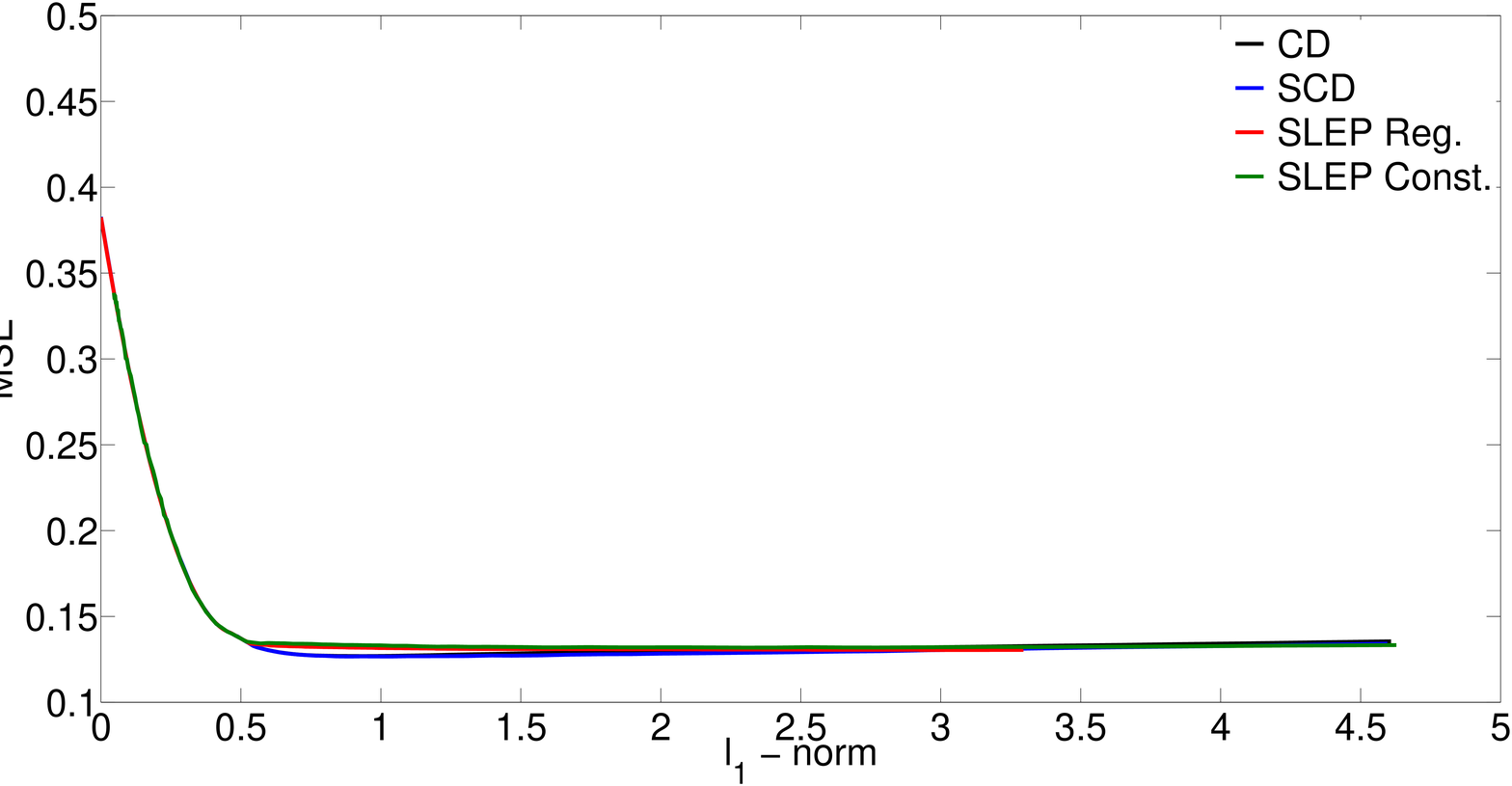}}}
\subfigure[]{
   {\includegraphics[scale = 0.1]{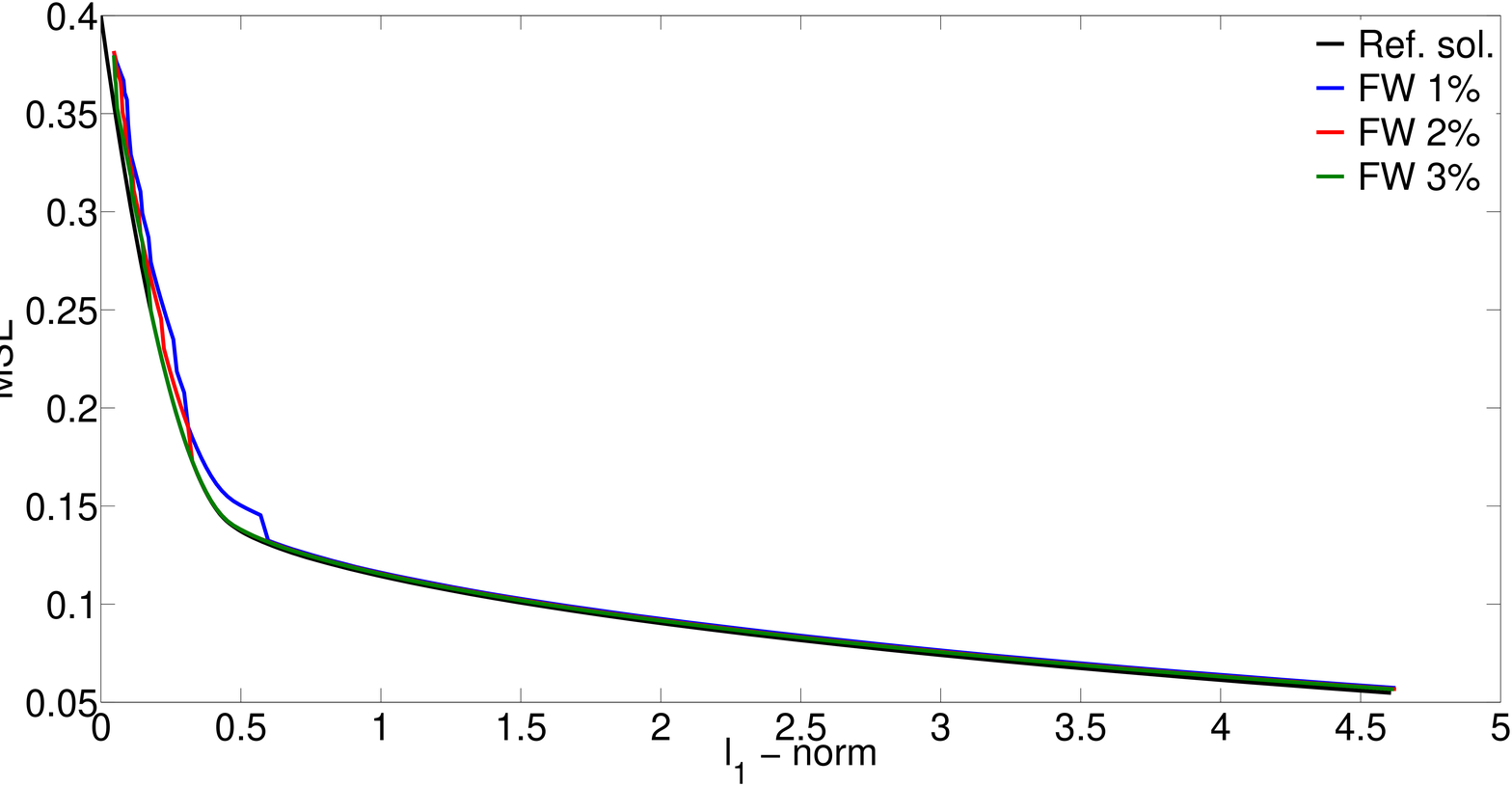}}}
\subfigure[]{
   {\includegraphics[scale = 0.1]{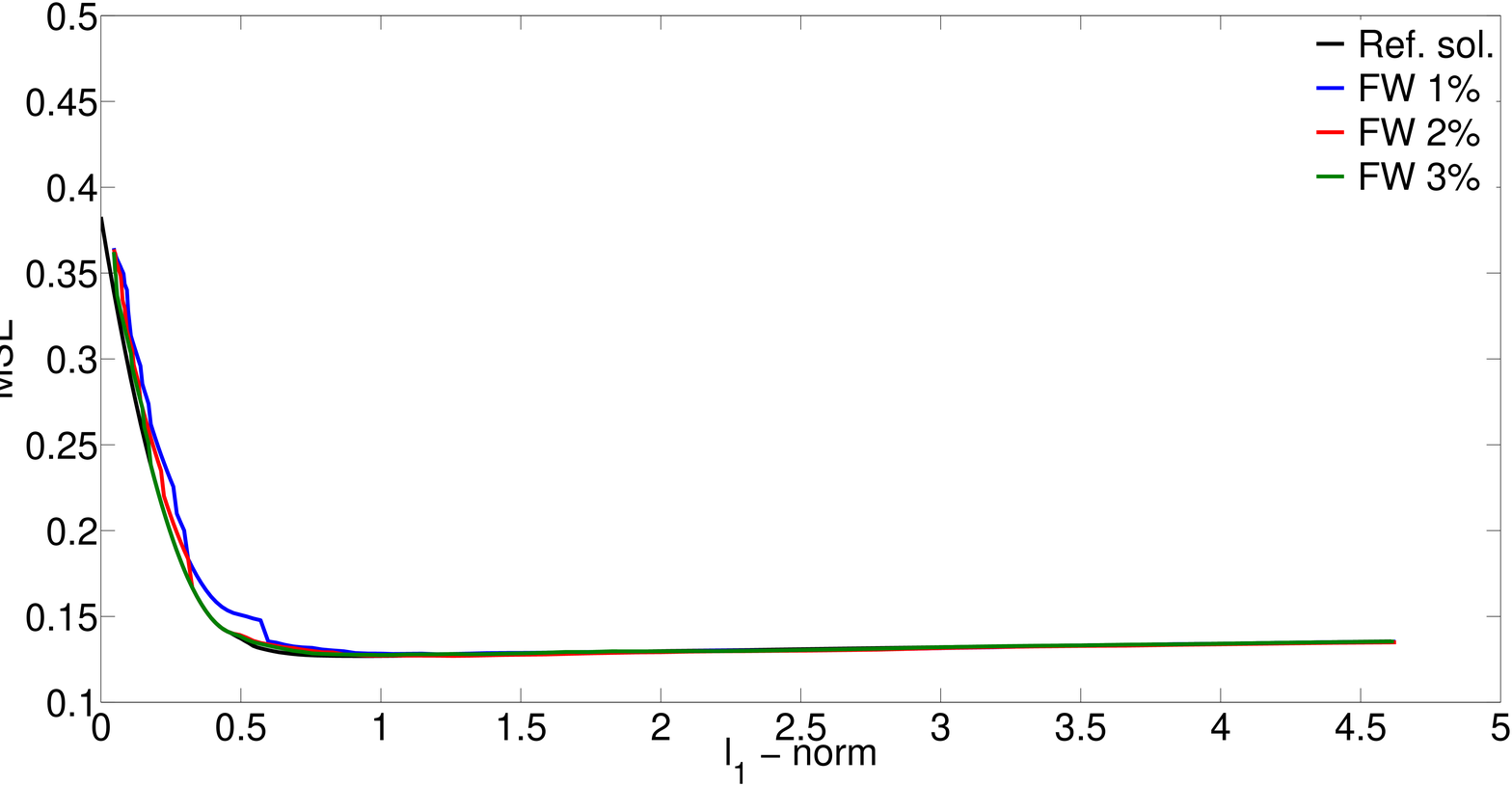}}}
\caption{\label{log1p_errors} Error curves ($\ell_1$-norm vs. MSE) for problem \textbf{E2006-log1p}: on top, training error (a) and test error (b) for CD, SCD and SLEP; on bottom, training error (c) and test error (d) for stochastic FW.}
\end{center}
\end{figure} 
First of all, we can see how the decrease in the objective value is basically identical in all cases, which indicates that with our sampling choices the use of a randomized algorithm does not affect the optimization accuracy. Second, the test error curves show that the predictive capability of all the FW models is competitive with that of the models found by the CD algorithm (particularly in the case of the larger problem \textbf{E2006-log1p}). It is also important to note that in all cases the best model, corresponding to the minimum of the test error curves, is found for a relatively low value of the constraint parameter, indicating that sparse solutions are preferable and that solutions involving more variables tend to cause overfitting, which is yet another incentive to use algorithms that can naturally induce sparsity. Again, it can be seen how the minima of all the curves coincide, indicating that all the algorithms are able to correctly identify the best compromise between sparsity and training error. The fact that we are able to attain the same models obtained by a state of the art algorithm such as Glmnet using a sampling size as small as $3\%$ of the total number of features is particularly noteworthy. Combined with the consistent advantages in CPU time over other competing solvers and its attractive sparsity properties, it shows how the randomized FW represents a solid, high-performance option for solving high-dimensional Lasso problems. 

%% file: LASSO_paper_Sec6_Conclusions.tex

\section{Conclusions and Perspectives}\label{conclusions}

In this paper, we have studied the practical advantages of using a randomized Frank-Wolfe algorithm to solve the constrained formulation of the Lasso regression problem on high-dimensional datasets involving a number of variables ranging from the hundred thousands to a few millions. We have presented a theoretical proof of convergence based on the expected value of the objective function. Our experiments show that we are able to obtain results that outperform those of other state-of-the-art solvers such as the Glmnet algorithm, a standard 
among practitioners, without sacrificing the accuracy of the model in a significant way. Importantly, our solutions are consistently more sparse than those found using several popular first-order methods, demonstrating the advantage of using an incremental, greedy optimization scheme in this context. 

In a future work, we intend to address the issue of whether it is possible to find suitable sampling conditions which can lead to a stronger stochastic convergence result, i.e. to certifiable probability bounds for approximate solutions. Finally, we remark that the proposed approach can be readily extended to other similar problems such as ElasticNet or more general $l_1$-regularized problems such as logistic regression, or to related applications such as the sparsification of SVM models. Another possibility to tackle various Lasso formulations is to exploit an equivalent formulation in terms of SVMs, an area where FW methods have already shown promising results. Together, all these elements strengthen the conclusions of our previous research work, showing that FW algorithms can provide a complete and flexible framework to efficiently solve a wide variety of large-scale Machine Learning and Data Mining problems.

%% file: LASSO_paper_Appendix_new.tex

\section*{Appendix: proof of Proposition 2}

Let $f$ be a convex differentiable real function on $\mathbb{R}^p$. Given $\cS \subseteq \{1,\ldots,p\}$, we define the \emph{restricted gradient} of $f$ with respect to $\cS$ as its scaled projection i.e. 
\begin{equation}\label{def:restricted_gradient}
\tilde \nabla_{\cS} f(\cdot) =  \frac{p}{\kappa}\Bigl(\sum_{i\in \mathcal{S}} \mathbf{e}_i \mathbf{e}_i^T \Bigr) \nabla f(\cdot)\,,
\end{equation}
where $\kappa = |\cS|$.
We define the \emph{curvature} constant of $f$ over a compact set $\Sigma$ as 
\begin{equation}\label{def:curvature}
C_{f} = \sup_{x\in\Sigma,y\in\Sigma_x} \frac{1}{\lambda^2} \left( f(y)-f(x)- (y-x)^T \nabla f(x) \right) \ ,
\end{equation}
where $\Sigma_x = \{y \in \Sigma: y= x + \lambda(s-x), s \in \Sigma, \lambda \in (0,1] \}$.

For any $\alpha \in \Sigma$ we define its primal gap and duality gap as 
\begin{align}
h(\alpha) &= f(\alpha) - f(\alpha^{\ast}) \label{def:primal_gap} \\
g(\alpha) &=  \max_{u\in\Sigma}\;(\alpha-u)^T\nabla f(\alpha) \label{def:dual_gap} \ ,
\end{align}
respectively. Convexity of the function $f$ implies that $f(\alpha)+(u-\alpha)^T\nabla f(\alpha)$ is nowhere greater than $f(\alpha)$. Therefore, 
\begin{equation} \label{eq:inequality_primal_dual_gaps}
g(\alpha) \geq h(\alpha) \;\; \forall \alpha \in \Sigma \ .
\end{equation}
For any iterate $\alpha^{(k)}$ generated by our algorithm, we define its \emph{expected primal gap and duality gap} as 
\begin{align}
h_{k+1} &= \mathbb{E}_{\:\mathcal{S}^{(k)}}\bigl[h(\alpha^{(k+1)})\bigr] \label{def:expected_primal_gap} \\
g_{k+1} &= \mathbb{E}_{\:\mathcal{S}^{(k)}}\bigl[g(\alpha^{(k+1)})\bigr] \label{def:expected_dual_gap} \ ,
\end{align}
respectively.  Here we denote by $\cS^{(k)}$ the random subset of $\{1,\ldots,p\}$ used to approximate the gradient at each iteration. Clearly,
\begin{equation}\label{eq:inequality_expected_primal_dual_gaps}
g_k \geq h_k \;\; \forall k \,.
\end{equation}

\begin{lemma}\label{main_lemma} Let $\alpha^{(k+1)}_\lambda = \alpha^{(k)} + \lambda(u^{(k)} - \alpha^{(k)})$ be a step in the direction of
\begin{equation}\label{eq:toward_vertex}
u^{(k)} \in \argmin_{u \,\in\, \Sigma} \, ( u-\alpha^{(k)} )^T \tilde \nabla_{\cS^{(k)}} f({\alpha}^{(k)}),
\end{equation}
with step-size $\lambda \in (0,1]$. Then
\begin{align} \label{eq:recurrence_main_lemma}
h_{k+1}(\lambda) &=\mathbb{E}_{\:\mathcal{S}^{(k)}}\bigl[h(\alpha^{(k+1)}_\lambda)\bigr] \leq h_{k} - \lambda g_k + \lambda^2 C_f
\end{align}
\begin{proof}
Since ${\alpha}^{(k)},u^{(k)} \in \Sigma$ and ${\alpha}^{(k+1)}_\lambda \in \Sigma_{{\alpha}^{(k)}}$, it follows from (\ref{def:curvature}) that
\begin{equation*}
f(\alpha^{(k+1)}_\lambda) \leq f(\alpha^{(k)}) + \lambda ( u^{(k)} -\alpha^{(k)} )^T \nabla f({\alpha}^{(k)}) +\lambda^2 C_{f} \ .
\end{equation*}
Expectation on both sides with respect to $\mathcal{S}^{(k)}$ yields 
\begin{equation}\label{eq:step_proof_main_lemma}
\begin{small}
\begin{aligned}
& \mathbb{E}_{\mathcal{S}^{(k)}}[f(\alpha^{(k+1)})] \leq f(\alpha^{(k)}) + \lambda\, \mathbb{E}_{\:\mathcal{S}^{(k)}}  [(u^{(k)}-\alpha^{(k)})^T \nabla f({\alpha}^{(k)})] +\lambda^2 {C}_f. 
\end{aligned}
\end{small}
\end{equation}
Since $\mathbb{E}_{\:\mathcal{S}^{(k)}} [ \tilde \nabla_{\cS} f(\alpha^{(k)})] = \nabla f(\alpha^{(k)})$, by the definition of $u^{(k)}$ and by the order preserving and linearity properties of expectation, we obtain

\begin{align} 
\mathbb{E}_{\:\mathcal{S}^{(k)}}  [ (u^{(k)}-\alpha^{(k)})^T \nabla f({\alpha}^{(k)})] & = \mathbb{E}_{\:\mathcal{S}^{(k)}} [\min_{u \,\in\, \Sigma} \,(u-\alpha^{(k)})^T \tilde \nabla_{\cS^{(k)}} f({\alpha}^{(k)})] \notag\\
& \leq \min_{u \,\in\, \Sigma} \, \mathbb{E}_{\:\mathcal{S}^{(k)}} [(u-\alpha^{(k)})^T\tilde \nabla_{\cS^{(k)}} f({\alpha}^{(k)})] \notag\\
& = \min_{u \,\in\, \Sigma} \,  (u-\alpha^{(k)})^T\nabla f({\alpha}^{(k)}) \notag\\
& = - g(\alpha^{(k)}) \label{eq:8} .
\end{align}
Substitution into~\eqref{eq:step_proof_main_lemma} and expectation with respect to $\cS^{(k-1)}$ finally yield
\begin{equation*}\label{eq:10}
\begin{aligned}
\mathbb{E}_{\:\mathcal{S}^{(k)}}[f(\alpha^{(k+1)}_\lambda)] \leq& \;\mathbb{E}_{\:\mathcal{S}^{(k-1)}} [f(\alpha^{(k)})] -\lambda\mathbb{E}_{\:\mathcal{S}^{(k-1)}} [g(\alpha^{(k)})]  +\lambda^2 {C}_f\,.
\end{aligned}     
\end{equation*}
Subtracting $f(\alpha^{\ast})$ from both sides, (\ref{def:expected_primal_gap}) and (\ref{def:expected_dual_gap}) yield the result.
\end{proof}
\end{lemma}

\begin{lemma}\label{lemma:global_bound} The initialization $\alpha^{(1)}=\bm{u}\ast$ with $\bu \ast \in \arg\min_{\bu \in \mathcal{V}(\Sigma)} f(\bu)$ guarantees $h_{k} \leq C_f \ \forall k > 1$. 
\end{lemma}
\begin{proof} 
First, note that $h_{k+1} \leq h_{k}$ $\forall k > 1$. Indeed, 
\begin{align*}
\min_{\lambda \in (0,1]} h(\alpha^{(k+1)}_\lambda) = \min_{\lambda \in (0,1]} h(\alpha^{(k)} + \lambda(u^{(k)} - \alpha^{(k)}) ) \leq h(\alpha^{(k)})\,.
\end{align*}
Thus
\begin{align*}
h_{k+1} = \mathbb{E}_{\:\mathcal{S}^{(k)}}\bigl[h(\alpha^{(k+1)}) \bigr] &= \mathbb{E}_{\:\mathcal{S}^{(k)}}\left[ \min_{\lambda \in (0,1]} h(\alpha^{(k+1)}_\lambda)\right]\\
&\leq  \mathbb{E}_{\:\mathcal{S}^{(k)}}\bigl[ h(\alpha^{(k)}) \bigr] = h_k\,.
\end{align*}
Now, from Lemma \ref{main_lemma}, any step in the direction of (\ref{eq:toward_vertex}) with step size $\lambda \in (0,1]$ satisfies
\begin{align*} 
h_{k+1}(\lambda) &=\mathbb{E}_{\:\mathcal{S}^{(k)}}\bigl[h(\alpha^{(k+1)}_\lambda \bigr] \leq h_{k} - \lambda g_k + \lambda^2 C_f  \leq h_{k} - \lambda h_k + \lambda^2 C_f \ .
\end{align*}
Suppose $h_{k} > C_f$. In this case, as $- \lambda h_k + \lambda^2 C_f < 0$, we can choose $\lambda=1$ to obtain
\begin{align*} 
\left.h_{k+1}(\lambda)\right|_{\lambda=1} &< h_{k} \ .
\end{align*}
But
\begin{align*}
\left.h_{k+1}(\lambda)\right|_{\lambda=1} =\mathbb{E}_{\:\mathcal{S}^{(k)}}\bigl[h(\bu^{(k)})\bigr] \leq \mathbb{E}_{\:\mathcal{S}^{(k)}}\bigl[h(\bu\ast)\bigr] = \mathbb{E}_{\:\mathcal{S}^{(k)}}\bigl[h(\alpha^{(1)})\bigr] = h_1\ .
\end{align*}
Thus, $h_1 < h_{k+1}$. This is a contradiction, since $h_{k+1} \leq h_{k}$ $\forall k > 1$.
\end{proof}

\begin{lemma}\label{lemma:recurrence_for_induction} At each iteration $k$ of Algorithm $2$,
\begin{align} \label{eq:recurrence_for_induction}
h_{k+1} &\leq h_{k} - \frac{h_{k}^2}{4C_f}\,.
\end{align}
\end{lemma}
\begin{proof}
At iteration $k$, Algorithm $2$ updates $\alpha^{(k)}$ by a line search in the direction of (\ref{eq:toward_vertex}). Hence
\begin{align} \label{eq:line_search}
h_{k+1} = \mathbb{E}_{\:\mathcal{S}^{(k)}}\bigl[h(\alpha^{(k+1)}) \bigr] =\mathbb{E}_{\:\mathcal{S}^{(k)}}\left[ \min_{\lambda \in (0,1]} h(\alpha^{(k+1)}_\lambda)\right]\,.
\end{align}

By the order preserving and linearity properties of expectation 
\begin{align} \label{eq:stefanos_lemma}
\mathbb{E}_{\:\mathcal{S}^{(k)}}\left[\min_{\lambda \in (0,1]} h(\alpha^{(k+1)}_\lambda)\right] \leq \min_{\lambda \in (0,1]} \mathbb{E}_{\mathcal{S}^{(k)}}\left[h(\alpha^{(k+1)}_\lambda)\right]\,.
\end{align}
From lemma \ref{main_lemma}, we have that any step in the direction of (\ref{eq:toward_vertex}) with step size $\lambda$ satisfies
\begin{align} \label{eq:previous_lemma}
h_{k+1}(\lambda) &=\mathbb{E}_{\:\mathcal{S}^{(k)}}\bigl[h(\alpha^{(k+1)}_\lambda \bigr] \leq h_{k} - \lambda g_k + \lambda^2 C_f \leq  h_{k} - \lambda h_k + \lambda^2 C_f \,.
\end{align}
Combining (\ref{eq:previous_lemma}) and (\ref{eq:stefanos_lemma}) produces
\begin{align} \label{eq:main_bound_new_lemma}
h_{k+1} = \min_{\lambda \in (0,1]} h_{k+1}(\lambda) &\leq \min_{\lambda \in (0,1]} \left( h_{k} - \lambda h_k + \lambda^2 C_f \right) \,.
\end{align}
From lemma \ref{lemma:global_bound}, $h_k < 2C_f$. Thus, the minimum at the right hand side is obtained for $\bar \lambda = h_k/2C_f$ (take derivative, equal to $0$, solve and check that $\bar \lambda<1$). Substituting this value of $\lambda$ yields
\begin{align} \label{eq:main_bound_new_lemma}
h_{k+1} &\leq  h_{k} - \frac{h_k^2}{2C_f} + \frac{h_k^2}{4C_f} = h_{k} - \frac{h_k^2}{4C_f}  \ .
\end{align}
\end{proof}

\noindent{\textbf{Proof of Proposition 2}.} With the above results in hand, we can now prove the convergence result in the main paper i.e.
\begin{equation*} 
h_k =  \mathbb{E}_{\mathcal{S}^{(k)}}\left[f(\alpha^{(k+1)})\right] - f(\alpha^{\ast}) \leq \frac{4C_f}{k+2}\,.
\end{equation*}
  
\begin{proof}
We prove the claim by induction on $k$. The base case $k = 0$ is trivial to verify from lemma \ref{lemma:global_bound} (as $4/3>1$). Now, from Lemma \ref{lemma:recurrence_for_induction} and the inductive hypothesis $h_k \leq \frac{4 {C}_f}{k+2}$, we obtain 
\begin{align*}\label{eq:induction}
h_{k+1} & \leq  h_{k} - \frac{h_k^2}{4C_f} \leq \frac{h_k}{1+ \frac{h_k}{4C_f}} = \frac{1}{\frac{1}{h_k} + \frac{1}{4C_f}} \leq \frac{1}{\frac{k+2}{4C_f} + \frac{1}{4C_f}} = \frac{4C_f}{(k+1)+2}\,.
\end{align*}
which completes the inductive step and yields the claimed bound.
\end{proof}